\documentclass[10pt]{article}

\usepackage[utf8]{inputenc}
\usepackage[T1]{fontenc}

\usepackage{epsf}
\usepackage{amsmath}

\allowdisplaybreaks

\usepackage[showframe=false]{geometry}
\usepackage{changepage}

\usepackage{epsfig}
\usepackage{amssymb}

\usepackage{amsthm}
\usepackage{setspace}
\usepackage{cite}
\usepackage{mcite}

\usepackage{algorithmic}  
\usepackage{algorithm}

\usepackage{shadow}
\usepackage{fancybox}
\usepackage{fancyhdr}

\usepackage{color}
\usepackage[usenames,dvipsnames,svgnames,table]{xcolor}
\newcommand{\bl}[1]{\textcolor{blue}{#1}}
\newcommand{\red}[1]{\textcolor{red}{#1}}

\definecolor{mypurple}{rgb}{.4,.0,.5}

\usepackage[hyphens]{url}

\usepackage[colorlinks=true,
            linkcolor=black,
            urlcolor=blue,
            citecolor=purple]{hyperref}

\usepackage{breakurl}

\def\y{{\bf y}}

\def\x{{\bf x}}

\def\x{{\mathbf x}}

\def\u{{\bf u}}

\def\x{{\bf x}}
\def\y{{\bf y}}
\def\z{{\bf z}}
\def\q{{\bf q}}

\def\c{{\bf c}}

\def\h{{\bf h}}

\def\cA{{\mathcal A}}
\def\cH{{\mathcal H}}

\def\be{\begin{equation}}
\def\ee{\end{equation}}
\def\ba{\left[\begin{array}}
\def\ea{\end{array}\right]}

\def\u{{\bf u}}

\def\x{{\bf x}}
\def\y{{\bf y}}
\def\z{{\bf z}}
\def\q{{\bf q}}

\def\c{{\bf c}}

\def\p{{\bf p}}

\def\1{{\bf 1}}

\def\0{{\bf 0}}

\def\erf{\mbox{erf}}
\def\erfc{\mbox{erfc}}

\def\calX{{\cal X}}
\def\calY{{\cal Y}}

\def\mR{{\mathbb R}}
\def\mN{{\mathbb N}}
\def\mE{{\mathbb E}}
\def\mS{{\mathbb S}}
\def\mP{{\mathbb P}}

\def\lp{\left (}
\def\rp{\right )}

\def\y{{\bf y}}

\def\x{{\bf x}}

\def\x{{\mathbf x}}

\def\u{{\bf u}}

\def\x{{\bf x}}
\def\y{{\bf y}}
\def\z{{\bf z}}
\def\q{{\bf q}}

\def\c{{\bf c}}

\def\h{{\bf h}}

\def\cH{{\cal H}}

\def\be{\begin{equation}}
\def\ee{\end{equation}}
\def\ba{\left[\begin{array}}
\def\ea{\end{array}\right]}

\def\u{{\bf u}}

\def\x{{\bf x}}
\def\y{{\bf y}}
\def\z{{\bf z}}
\def\q{{\bf q}}

\def\c{{\bf c}}

\def\p{{\bf p}}

\def\({\left (}
\def\){\right )}

\def\1{{\bf 1}}

\def\q{{\bf q}}

\def\0{{\bf 0}}

\def\cX{{\mathcal X}}
\def\cY{{\mathcal Y}}

\usepackage{xcolor}
\usepackage{color}

\definecolor{darkgreen}{rgb}{0, 0.4,0}

\definecolor{purplebrown}{rgb}{0.5,0.1,0.6}

\definecolor{ultclupcol}{rgb}{0.1,0.5,0.5}

\definecolor{mytrycolor}{rgb}{0.5,0.7,0.2}

\definecolor{ultclupcola}{rgb}{.5,0,.5}

\definecolor{shadebrown}{rgb}{0.1,0.1,0.9}
\definecolor{lightblue}{rgb}{0.2,0,1}

\usepackage{fancybox}
\usepackage{graphicx}
\usepackage{epstopdf}
\usepackage{epsfig}
\usepackage{wrapfig}
\usepackage{subfigure}

\usepackage{xcolor}
\usepackage{tcolorbox}
\tcbuselibrary{skins}

\newtcbox{\xmybox}{on line,
arc=7pt,
before upper={\rule[-3pt]{0pt}{10pt}},boxrule=0pt,
boxsep=0pt,left=6pt,right=6pt,top=0pt,bottom=0pt,enhanced, coltext=blue, colback=white!10!yellow}

\newtcbox{\xmyboxa}{on line,
arc=7pt,
before upper={\rule[-3pt]{0pt}{10pt}},boxrule=0pt,
boxsep=0pt,left=6pt,right=6pt,top=0pt,bottom=0pt,enhanced, colback=white!10!yellow}

\newtcbox{\xmyboxb}{on line,
arc=7pt,
before upper={\rule[-3pt]{0pt}{10pt}},boxrule=1pt,colframe=darkgreen!100!blue,
boxsep=0pt,left=6pt,right=6pt,top=0pt,bottom=0pt,enhanced, colback=white!10!yellow}

\newtcbox{\xmyboxc}{on line,
arc=7pt,
before upper={\rule[-3pt]{0pt}{10pt}},boxrule=.7pt,colframe=blue!100!blue,
boxsep=0pt,left=6pt,right=6pt,top=0pt,bottom=0pt,enhanced, coltext=blue, colback=white!10!yellow}

\newtcbox{\xmytboxa}{on line,
arc=7pt,
before upper={\rule[-3pt]{0pt}{10pt}},boxrule=.0pt,colframe=pink!50!yellow,
boxsep=0pt,left=6pt,right=6pt,top=0pt,bottom=0pt,enhanced, coltext=white, colback=blue!40!red}

\newtcbox{\xmytboxb}{on line,
arc=7pt,
before upper={\rule[-3pt]{0pt}{10pt}},boxrule=.0pt,colframe=pink!50!yellow,
boxsep=0pt,left=6pt,right=6pt,top=0pt,bottom=0pt,enhanced, coltext=white, colback=white!40!green}

\makeatletter
\newcommand\subsubsubsection{\@startsection{paragraph}{4}{\z@}{-2.5ex\@plus -1ex \@minus -.25ex}{1.25ex \@plus .25ex}{\normalfont\normalsize\bfseries}}
\newcommand\subsubsubsubsection{\@startsection{subparagraph}{5}{\z@}{-2.5ex\@plus -1ex \@minus -.25ex}{1.25ex \@plus .25ex}{\normalfont\normalsize\bfseries}}
\makeatother

\newtheorem{theorem}{Theorem}

\newtheorem{corollary}{Corollary}

\begin{document}

\begin{singlespace}

\title {Injectivity capacity of ReLU gates  
}
\author{
\textsc{Mihailo Stojnic
\footnote{e-mail: {\tt flatoyer@gmail.com}} }}
\date{}
\maketitle

\centerline{{\bf Abstract}} \vspace*{0.1in}

We consider the injectivity property of the ReLU networks layers.  Determining the ReLU injectivity capacity (ratio of the number of layer's inputs and outputs) is established as isomorphic to determining the capacity of the so-called $\ell_0$ spherical perceptron. Employing \emph{fully lifted random duality theory} (fl RDT)  a powerful program is developed and utilized to handle the $\ell_0$ spherical perceptron and implicitly the ReLU layers injectivity. To put the entire fl RDT machinery in  practical use, a sizeable set of numerical evaluations is conducted as well. The lifting mechanism is observed to converge remarkably fast with relative corrections in the estimated quantities not exceeding $\sim 0.1\%$ already on the third level of lifting. Closed form explicit analytical relations among key lifting parameters are uncovered as well. In addition to being of incredible importance in handling all the required numerical work, these relations  also  shed a new light on  beautiful parametric interconnections within the lifting structure. Finally, the obtained results are also shown to fairly closely match the replica predictions from \cite{MBBDN23}.

\vspace*{0.25in} \noindent {\bf Index Terms: Injectivity; ReLU networks; Random duality}.

\end{singlespace}

\section{Introduction}
\label{sec:back}

Over the last 15-20 years we have been witnessing a rapid development of machine learning (ML) and neural networks (NN) concepts. As the need for efficient processing and interpretation of large data sets is estimated to further grow in the years to come, many fundamental algorithmic and theoretical NN breakthroughs are to be expected. To be able to adequately address upcoming challenges an excellent understanding of the ultimate limits of the employed technologies is needed. We in this paper study a mathematical problem that is directly connected to a notion of network capacity which is an example of such a limit. 

Characterizing presence or absence of \emph{injectivity} as a property of random functions is the mathematical problem of our interest here. The mere definition of the functional injectivity implies its critical role in studying inverse problems. Namely, well- or ill-posedness of these problems is in a direct correspondence with the associated injectivity. Recent utilization of neural networks in studying (nonlinear) inverse problems therefore critically relies on their injectivity properties (see, e.g., \cite{ArridgeMOS19,BoraJPD17,KothariKHD21,HandLV18,DharGE18,LeiJDD19,EstrachSL14,DaskalakisRZ20a}) . Consequently, injectivity as a purely mathematical object is in these contexts transformed into a practically rather important NN architectures feature. Various NN injectivity aspects (related or unrelated to inverse problems) have been of interest in recent years (see, e.g., \cite{FazlyabRHMP19,GoukFPC21,JordanD20} for stability Lipshitzian/injectivity properties, \cite{PuthawalaKLDH22,PuthawalaLDH22,RossC21} for a role of the injective ReLU nets in approximative maps and manifold densities,  \cite{FletcherRS18,LeiJDD19,RomanoEM17,HeWJ17,MardaniSDPMVP18,SchniterRF16,RanganSF17,ShahH18,HegdeWB07,HandLV18,WuRL19}  for algorithmic approaches to (deep and/or linear) generative models,
  and \cite{MBBDN23,PenningtonW17,LouartLC17} and references therein for rather interesting random matrix - random neural networks perspectives).

To understand the injectivity of a whole network it is typically desirable to understand the injectivity of a subset of its layers. We here take a first step in that direction and consider a single layer of a ReLU network (clearly, in single layer networks it would actually be the whole network). Of particular interest is the so-called injectivity \emph{capacity}, defined as the ratio of the layer's input and output size. Although conceptually somewhat connected, these capacities are actually different from the ones studied in classical spherical perceptrons related literature
\cite{Schlafli,Cover65,Winder,Winder61,Wendel,Cameron60,Joseph60,Gar88,Ven86,BalVen87}. The classical capacities typically relate to the network's ability to memorize/store patterns. Throughout the presentation below, it will be shown that the injectivity capacities can in fact be reformulated to become capacities of variants of classical perceptrons. We refer to such variant as $\ell_0$ \emph{spherical perceptrons}.

Majority of the above injectivity related works rely on so-called \emph{qualitative} performance characterizations. These usually give correct orders of magnitude of the studied quantities and provide a good guidance regarding the overall abilities and usefulness of the underlying methodologies. As we will be interested in more precise, i.e.,  \emph{quantitative} performance characterizations, more closely related to our work are results from \cite{Pal21,Clum22,MBBDN23,PuthawalaKLDH22}. They are discussed throughout the presentation after we introduce all the needed technical prerequisites. Before proceeding  with the detailed technical considerations, we in the following section present a summary of our key results.

\subsection{Our contributions}
\label{sec:contrib}

As stated above, given the recent popularity of machine learning and consequential importance of understanding key neural networks features, we study the injectivity as a property particularly associated with ReLU gates. The focus is on single layers (as mentioned above, understanding them is a key building block if one aims to eventually understand more complex architectures). The results summarized below are obtained for fully connected feed forward architecture in a statistical so-called \emph{linear/proportional} high-dimensional regime (the layer's number of inputs is $\alpha$ times larger than its number of outputs; $\alpha$ is precisely the capacity and remains constant as the number of inputs/outputs grows).

\begin{itemize}
  \item Following into the footsteps of \cite{StojnicGardGen13,StojnicGardSphErr13,Stojnicnegsphflrdt23}, we recognize the connection between studying ReLU injectivity properties and random feasibility problem (rfps). In particular (see Section \ref{sec:mathsetup}), we observe
      \begin{eqnarray}      \label{eq:eqint1}
      \mbox{ReLU injectivity} \qquad \Longleftrightarrow \qquad \ell_0 \quad \mbox{spherical perceptron},
      \end{eqnarray}
      which then implies
      \begin{eqnarray}      \label{eq:eqint2}
      \mbox{ReLU injectivity capacity} \qquad = \qquad \ell_0 \quad \mbox{spherical perceptron capacity}.
      \end{eqnarray}
  \item Studying rfps is then related to studying \emph{bilinearly index} (bli) random processes. Such a link allows us to use a recent progress in studying bli's and in particular to utilize the \emph{fully lifted random duality theory} (fl RDT) (see Section \ref{sec:randlincons}).
  \item Utilizing fl RDT we introduce an innovative method to establish a generic program for studying ReLU gates and their injectivity properties. As a results, we provide a precise ReLU injectivity capacity characterizations for any level of lifting (see Section \ref{sec:prac}).
 \item   For the first three lifting levels we provide explicit numerical values of the capacity as well as of all associated fl RDT parameters (see Section \ref{sec:nuemricalags} and Tables \ref{tab:tab1}, \ref{tab:tab1p}, \ref{tab:tab2}, and \ref{tab:tab3}).
  \item To ensure fl RDT's practical relevance, a sizeable set of numerical evaluations is needed. After conducting them, we uncover an astonishingly fast convergence of lifting mechanism with relative  corrections in the estimated quantities not exceeding $\sim 0.1\%$ already
     on the third level of lifting (see Table \ref{tab:tab3}).
\item A remarkable set of closed form analytical relations among  key lifting parameters is uncovered as well. These are not only extremely helpful in conducting all the required numerical work, but also of critical importance in providing direct insights into a rather beautiful  structuring of the parametric interconnections (see Corollaries \ref{cor:closedformrel1} and \ref{cor:3closedformrel1}).
\end{itemize}

\section{Injectivity capacity -- mathematical setup}
 \label{sec:mathsetup}

In what follows we study the following \emph{nonlinear} system of equations
\begin{eqnarray}
\y=f_{g}(A\bar{\x}), \label{eq:ex1}
\end{eqnarray}
where $A\in\mR^{m\times n}$ is a linearly transformational system matrix, $f_{g}(\cdot): \mR^m\rightarrow\mR^m$ is a (nonlinear) system function, and $\bar{\x}\in\mR^n$ is a vector to be recovered. Both the methodology and the results that we will present are fairly generic and can be applied to a host of different system functions. For the concreteness, we consider the componentwise functions that act in the same manner on each coordinate of their vector arguments. Given the importance of the underlying studies in the context of neural networks, we take ReLU action as a concrete example for which we conduct all evaluations. In other words, we take
\begin{eqnarray}
 f_{g}(\x)=\max(\x,0), \label{eq:ex1a0}
\end{eqnarray}
 with $\max$ being applied componentwise. Consequently, the nonlinear system of our interest, (\ref{eq:ex1}), becomes
\begin{eqnarray}
\hspace{-1in}\bl{\textbf{\emph{ReLU system:}}} \qquad \qquad  \y=\max(A\bar{\x}). \label{eq:ex1a2}
\end{eqnarray}
We will consider high-dimensional so-called \emph{linear} (or \emph{proportional}) regimes with
 \begin{eqnarray}
\alpha  \triangleq    \lim_{n\rightarrow \infty} \frac{m}{n}, \label{eq:ex15}
\end{eqnarray}
and $\alpha\in[0,1]$ and $\beta\in[0,1]$.


The ReLU gates (mathematically represented through (\ref{eq:ex1a2})) have played an important role since the very early days of machine learning (ML) and neural networks (NN). Due to enormous popularity of these fields in recent years, the interest in ReLU concepts picked up  and they have been the subject of intensive conceptual and algorithmic studies over the better portion of the last decade. Their properties and ultimate usefulness have been considered in both more complex NN architectures (see, e.g., \cite{BalMalZech19,ZavPeh21,AMZ24,Stojnictcmspnncapdiffactrdt23,Stojnictcmspnncapdinfdiffactrdt23}) as well as in basic single layer structures (see, e.g., \cite{PuthawalaKLDH22,FuruyaPLH23,MBBDN23,PuthawalaLDH22,Pal21,Clum22}). For example, within a typical ML context the above system (\ref{eq:ex1a2}) would correspond to a layer of a feed forward neural net with input $\bar{\x}$, output $\y$, and the weights being the rows of $A$. As a side note, we mention that the above mathematical model (together with all the results that we present below) is applicable in a slightly different context as well. Namely, if one takes $\bar{\x}$ as the desired weight vector and $A$ as the sample of $m$ data vectors then (\ref{eq:ex1a2}) represents functioning of one ReLU gate on $m$ data points (akin to, say, classifying functioning of standard perceptrons). When viewed in such a context, the injectivity of $f_g(\cdot)$ corresponds to the ability to \emph{uniquely} train the gate while its capacity corresponds to the ratio of the minimal, injectivity ensuring, sample complexity and cardinality of the weights set.

A particular feature that makes ReLU activations a bit different from, say, typical \emph{sign} perceptrons is that, in addition to their excellent nonlinear architecture building abilities, they also allow for the \emph{exact} reconstruction of the input $\bar{\x}$. This is easily seen after recognizing that existence of at least $n$ nonzero elements in $\y$ is sufficient for such a task (we assume a non-degenerative scenario where any subset of $n$ rows of $A$ is of full rank, which in statistical scenarios considered throughout the paper is typically true with probability approaching $1$ as $n\rightarrow\infty$). While it is relatively easy to observe that ReLU system allows for a perfect recovery of $\bar{\x}$, it is highly nontrivial to determine the minimal length of $\y$ that suffices for such a successful recovery (see, e.g., \cite{PuthawalaKLDH22,FuruyaPLH23,MBBDN23,PuthawalaLDH22,Pal21,Clum22}). \emph{Precisely} determining  minimal $m$ is the main concern of this paper. In the high-dimensional linear regime that effectively means determining the minimal $\alpha$. Following the common terminology (see, e.g., \cite{PuthawalaKLDH22,FuruyaPLH23,MBBDN23,PuthawalaLDH22}), we, for a statistical $A$, formally define such a value as
 \begin{eqnarray}
\bl{\emph{\textbf{ReLU injectivity capacity:}}} & & \nonumber   \\
\hspace{-0in} \alpha_{ReLU}^{(inj)}  \triangleq  \min & & \alpha\nonumber \\
  \mbox{subject to} & &  \lim_{n\rightarrow\infty}\mP_A(\forall\bar{\x},\nexists \x\neq \bar{\x} \quad \mbox{such that}\quad \max(A\bar{\x},0)=\max(A\x,0))=1. \label{eq:ex1a3}
\end{eqnarray}
Given the above definition, the following observations turn out to be crucial. Relying on the random feasibility considerations established in \cite{StojnicGardGen13,StojnicGardSphErr13,StojnicDiscPercp13,Stojnicbinperflrdt23,Stojnicnegsphflrdt23}, one first recognizes that the following \emph{feasibility} optimization problem is of key interest and directly relates to the condition under probability
 \begin{eqnarray}
 {\mathcal F}(A,\alpha): \qquad \qquad    \mbox{find} & & \x\nonumber \\
  \mbox{subject to} & &  A\x=\z\nonumber \\
  & & \|\z\|_0<n. \label{eq:ex1a4}
\end{eqnarray}
To see the connection, one notes that if there is no solution to the above feasibility problem then $\y$ in (\ref{eq:ex1a2}) must have at least $n$ nonzero components which (under the non-degenerative assumption stated a bit earlier) is sufficient to \emph{uniquely} recover $\bar{\x}\in\mR^n$. As optimization in (\ref{eq:ex1a4}) is scaling invariant, we can without loss of generality assume that $\x$ is restricted to the unit sphere, i.e., we can assume that optimization runs over $\|\x\|_2=1$. Also, following further the trends of \cite{StojnicGardGen13,StojnicGardSphErr13,StojnicDiscPercp13,Stojnicbinperflrdt23,Stojnicnegsphflrdt23}, we can introduce an artificial objective $f_a(\x,\z):\mR^{n+m}\rightarrow\mR$ and transform the (\emph{random}) feasibility problem  (rfp) (\ref{eq:ex1a4}) into the following (\emph{random}) optimization problem
 \begin{eqnarray}
  \hspace{-1.8in}\bl{\ell_0 \quad \textbf{\mbox{\emph{spherical perceptron}}}} \qquad   \qquad   \qquad    \min_{\x,\z} & & f_a(\x,\z)\nonumber \\
  \mbox{subject to} & &  A\x=\z\nonumber \\
  & & \|\z\|_0<n \nonumber \\
  & & \|\x\|_2=1. \label{eq:ex1a5}
\end{eqnarray}
 The necessary precondition for solvability of any optimization problem is that it is actually feasible. Under the feasibility assumption, we can rewrite  (\ref{eq:ex1a5})  as
\begin{eqnarray}
\xi_{feas}^{(0)}(f_a) = \min_{\|\x\|_2=1,\|\z\|_0<n} \max_{\y\in\cY}  \lp f_a(\x,\z) +\y^T A\x -\y^T  \z  \rp,
 \label{eq:ex2}
\end{eqnarray}
where $\cY=\mR^m$. Since $f_a(\x,\z)=0$ is clearly an artificial object, we can specialize back to $f_a(\x,\z)=0$ and obtain
\begin{eqnarray}
\xi_{feas}^{(0)}(0) = \min_{\|\x\|_2=1,\|\z\|_0<n} \max_{\y\in\cY}  \lp \y^T A\x -\y^T  \z  \rp.
 \label{eq:ex2a1}
\end{eqnarray}
The main point behind the connection to rfps is now easily seen from (\ref{eq:ex2a1}). One first observes that the existence of an $\|\x\|_2=1$ and $\|\z\|_0<n$  such that $A\x=\z$, i.e., such that (\ref{eq:ex1a4}) is feasible, ensures that (\ref{eq:ex2a1})'s inner maximization  can do no better than make $\xi_{feas}^{(0)}(0) =0$. If such an $\x$ and $\z$ do not exist, then at least one of the equalities in $A\x= \z$ is not satisfied which allows the inner maximization to trivially make $\xi_{feas}^{(0)}(0) =\infty$. As  $\xi_{feas}^{(0)}(0) =\infty$ and $\xi_{feas}^{(0)}(0) >0$ are equivalent from the feasibility point of view, one has that the underlying optimization in (\ref{eq:ex2a1}) is for all practical purposes $\y$ scaling  insensitive. This further allows restriction to $\|\y\|_2=1$  which in return ensures boundedness of $\xi_{feas}^{(0)}(0)$. Keeping all of this in mind, one recognizes that determining
\begin{eqnarray}
\xi_{feas}(0)
& =  &
\min_{\|\x\|_2=1,\|\z\|_0<n} \max_{\|\y\|_2=1}   \lp \y^TA\x -\y^T\z \rp,
 \label{eq:ex3}
\end{eqnarray}
is of critical importance for the analytical characterization of (\ref{eq:ex1a4}). In particular, the sign of the objective in (\ref{eq:ex3}) (i.e., of $\xi_{feas}(0)$) determines (\ref{eq:ex1a4})'s  feasibility. More concretely, (\ref{eq:ex1a4}) is infeasible if  $\xi_{feas}(0)>0$ and feasible if   $\xi_{feas}(0)\leq 0$. All of the above holds deterministically, i.e., for any $A$ and then also automatically for random $A$ (such as the one comprised of iid standard normal components that we consider here). After a combination with (\ref{eq:ex1a4}) and (\ref{eq:ex3}), one can, for the completeness, rewrite (\ref{eq:ex1a3}) as
 \begin{eqnarray}
\hspace{-0in} \alpha_{ReLU}^{(inj)} & \triangleq  &
 \max \{\alpha |\hspace{.08in}
 \lim_{n\rightarrow\infty}\mP_A(\nexists \x\neq \bar{\x} \quad \mbox{such that}\quad \max(A\bar{\x},0)=\max(A\x,0))=1\}
  \nonumber \\
  & = & \max \{\alpha |\hspace{.08in}  \lim_{n\rightarrow\infty}\mP_A\lp{\mathcal F}(A,\alpha) \hspace{.07in}\mbox{is feasible} \rp\longrightarrow 1\}
  \nonumber \\
  & = & \max \{\alpha |\hspace{.08in}  \lim_{n\rightarrow\infty}\mP_A\lp\xi_{ReLU}(0)\triangleq \xi_{feas}(0)>0\rp\longrightarrow 1\}.
  \label{eq:ex4}
\end{eqnarray}
As stated earlier, the above is the \emph{statistical} ReLU injectivity capacity (the deterministic  variant is the same quantity without $\mP_A$). Throughout the presentation, we adopt the convention where the subscripts next to $\mP$ and $\mE$ denote the underlying randomness. To shorten writing, these subscripts are left unspecified when the underlying randomness is clear from the context. Along the same lines, the term capacity is regularly used instead of \emph{statistical} capacity.

\subsection{Related technical results}
\label{sec:relwork}

Connecting the ReLU injectivity to classical studies of perceptrons turned out to be fairly fruitful. Among the most successful perceptron characterizations date back to the sixties of the last century and the pioneering works of Cover, Wendel, and Winder \cite{Cover65,Wendel,Winder,Winder61}. Certainly the most celebrated among many of the great results that they established is the so-called capacity of the spherical perceptron and in particular its property that it asymptotically doubles the ambient dimension. These early considerations were massively generalized by Schcherbina and Tirozzi \cite{SchTir02,SchTir03} to convexly  thresholded variants (see, also \cite{Talbook11a}) and then by Stojnic \cite{StojnicGardGen13,StojnicGardSphErr13,Stojnicnegsphflrdt23} to both convex and non-convex ones as well (a great set of results has been obtained in parallel for various other types of preceptrons; see, e.g., \cite{Talbook11a,Talbook11b,BoltNakSunXu22,DingSun19,NakSun23,Tal99a,StojnicDiscPercp13,Stojnicbinperflrdt23} and references therein for closely related \emph{binary} perceptrons, \cite{AbbLiSly21b,PerkXu21,AbbLiSly21a,AlwLiuSaw21,AubPerZde19,GamKizPerXu22} for the corresponding \emph{symmetric binary} ones, and also \cite{FPSUZ17,Gar88,GarDer88,GutSte90,KraMez89} for their associated statistical physics replica methods based considerations).

Combining union bounding with spherical perceptron capacity characterizations \cite{Cover65,Wendel,Winder,Winder61},  \cite{PuthawalaKLDH22} obtained $\approx 3.3$ and $\approx 10.5$  as respective lower and upper bounds on the capacity under the Gaussianity of $A$ (for an upper bound decrease to $\approx 9.091$ see \cite{MBBDN23} and reference therein  \cite{Pal21,Clum22}; on the other hand, if one is allowed to optimally choose the weights (elements of $A$), \cite{PuthawalaKLDH22} determined the capacity to be $2$). The mere fact that there was a discrepancy between the bounds was already a strong signal that the bounds are likely to be loose. Opting for union bounding is not surprising as the underlying problems are highly nonconvex.  However, as discussed in a long line of work on various models of perceptrons \cite{Tal06,Talbook11a,StojnicGardSphErr13,StojnicDiscPercp13}, facing non-convexity usually implies that application of classical techniques typically results in a deviation from the exactness. Different strategies were tried in \cite{Pal21,Clum22} where the problem is connected to intersecting random subspaces studies in high-dimensional integral geometry. Relying  on the kinematic formula \cite{SW08} (and in a way resembling earlier compressed sensing approaches of  \cite{ALMT14,DonohoPol}), \cite{Pal21} viewed the injectivity capacity through a heuristical Euler characteristics approximation approach and obtained a bit better capacity prediction $\approx 8.34$. Attempting to check the exactness of this prediction, \cite{MBBDN23} followed further into the  \cite{Pal21,Clum22}'s footsteps of connecting studying ReLU injectivity and intersecting random subspaces and showed how evoking Gordon's escape through a mesh theorem \cite{Gordon85,Gordon88} can result in a rather large capacity bound of $\approx 23.54$. The authors of \cite{MBBDN23} then utilized the (plain) \emph{random duality theory} (RDT) that Stojnic invented in a long line of work \cite{StojnicCSetam09,StojnicICASSP10var,StojnicGorEx10,StojnicRegRndDlt10,StojnicGardGen13} and obtained capacity upper-bound $\approx 7.65$ which was both much better than what they could get by the Gordon's theorem and also sufficiently better to refute the above mentioned  conjecture obtained through heuristic Euler characteristics approximation. \cite{MBBDN23}  proceeded further by considering the statistical physics replica theory methods and obtained capacity predictions based on them. Throughout the presentation we revisit these and discuss how they relate to the results that we obtain here.

\subsection{Free energy view}
 \label{sec:feeeng}

To smoothen the transition to the analysis that follows, we find it convenient to connect the above random feasibility/optimization concepts to statistical mechanics  free energies. Since the above anticipates the importance of $\xi_{feas}(0)$, it is natural to focus on its representation from (\ref{eq:ex3}). As discussed in detail in,e.g.,
\cite{StojnicGardGen13,Stojnicnegsphflrdt23}, studying the formulation from (\ref{eq:ex3}) is analogous to studying \emph{free energies} in statistical mechanics. For completeness we below sketch the contours of such an analogy  (more thorough discussions can be found in, e.g., \cite{StojnicGardGen13,Stojnicnegsphflrdt23} and references therein).

We start with the \emph{bilinear Hamiltonian}
\begin{equation}
\cH_{sq}(A)= \y^TA\x,\label{eq:ham1}
\end{equation}
and associate with it the following (reciprocal form of) partition function
\begin{equation}
Z_{sq}(\beta,A)=\sum_{\x\in\cX} \lp \sum_{\y\in\cY}e^{\beta\cH_{sq}(A)} \rp^{-1},  \label{eq:partfun}
\end{equation}
with $\cX\subset \mR^n$ and $\cY\subset \mR^n$ being two general sets. In the thermodynamic limit the corresponding (reciprocal form of) \emph{average} free energy is
\begin{eqnarray}
f_{sq}(\beta) & = & - \lim_{n\rightarrow\infty}\frac{\mE_A\log{(Z_{sq}(\beta,A)})}{\beta \sqrt{n}}
=\lim_{n\rightarrow\infty} \frac{\mE_A\log\lp \sum_{\x\in\cX} \lp \sum_{\y\in\cY}e^{\beta\cH_{sq}(A)} \rp^{-1} \rp}{\beta \sqrt{n}} \nonumber \\
& = &  - \lim_{n\rightarrow\infty} \frac{\mE_A\log\lp \sum_{\x\in\cX} \lp \sum_{\y\in\cY}e^{\beta\y^T A \x)}\rp^{-1} \rp}{\beta \sqrt{n}}.\label{eq:logpartfunsqrt}
\end{eqnarray}
Specializing to  ``zero-temperature'' ($T\rightarrow 0$ or  $\beta=\frac{1}{T}\rightarrow\infty$) regime, one further finds the so-called (average) \emph{ground-state} energy
\begin{eqnarray}
f_{sq}(\infty)    \triangleq    \lim_{\beta\rightarrow\infty}f_{sq}(\beta) & = &
 - \lim_{\beta,n\rightarrow\infty}\frac{\mE_A\log{(Z_{sq}(\beta,A)})}{\beta \sqrt{n}}
 \nonumber \\
& = &
-  \lim_{n\rightarrow\infty}\frac{\mE_A \max_{\x\in\cX} -  \max_{\y\in\cY} \y^T A\x}{\sqrt{n}}
=
 \lim_{n\rightarrow\infty}\frac{\mE_A \min_{\x\in\cX}  \max_{\y\in\cY} \y^T A\x}{\sqrt{n}}. \nonumber \\
  \label{eq:limlogpartfunsqrta0}
\end{eqnarray}
As (\ref{eq:limlogpartfunsqrta0}) is directly related to (\ref{eq:ex3}), $f_{sq}(\infty)$ is very tightly connected to $\xi_{feas}(0)$. This implies that studying objects fairly similar to $f_{sq}(\infty)$  and understanding their behavior is critically important for analytical characterization of  (\ref{eq:ex4}). However, a direct studying of $f_{sq}(\infty)$ is typically hard. We instead first study $f_{sq}(\beta)$ assuming a general $\beta$  and then specialize  to $\beta\rightarrow\infty$ regime. The  presentation is on occasion further facilitated by neglecting analytical details that remain of no relevance in the ground state regime.

\section{Connection to bli random processes and sfl RDT}
\label{sec:randlincons}

One observes that the free energy from (\ref{eq:logpartfunsqrt}),
\begin{eqnarray}
f_{sq}(\beta) & = &- \lim_{n\rightarrow\infty} \frac{\mE_A\log\lp \sum_{\x\in\cX} \lp \sum_{\y\in\cY}e^{\beta\y^T A \x)}\rp^{-1} \rp}{\beta \sqrt{n}},\label{eq:hmsfl1}
\end{eqnarray}
 can be viewed as a function of \emph{bilinearly indexed} (bli) random process $\y^TG\x$. Such  an observation is  of key importance as it enables connecting $f_{sq}(\beta)$ to studies of bli processes presented in \cite{Stojnicsflgscompyx23,Stojnicnflgscompyx23,Stojnicflrdt23}. The path traced in \cite{Stojnicnegsphflrdt23} is particularly convenient for establishing such a connection. We start things off by introducing a collection of necessary technical definitions. For $r\in\mN$, $k\in\{1,2,\dots,r+1\}$, sets $\cX\subseteq \mR^{n}$ and $\cY\subseteq \mR^m$, function $f_S(\cdot):\mR^{n}\rightarrow R$, real scalars $x$, $y$, and $s$  such that $x>0$, $y>0$, and $s^2=1$,   vectors $\p=[\p_0,\p_1,\dots,\p_{r+1}]$, $\q=[\q_0,\q_1,\dots,\q_{r+1}]$, and $\c=[\c_0,\c_1,\dots,\c_{r+1}]$ such that
 \begin{eqnarray}\label{eq:hmsfl2}
1=\p_0\geq \p_1\geq \p_2\geq \dots \geq \p_r\geq \p_{r+1} & = & 0 \nonumber \\
1=\q_0\geq \q_1\geq \q_2\geq \dots \geq \q_r\geq \q_{r+1} & = &  0,
 \end{eqnarray}
$\c_0=1$, $\c_{r+1}=0$, and ${\mathcal U}_k\triangleq [u^{(4,k)},\u^{(2,k)},\h^{(k)}]$  with  $u^{(4,k)}\in\mR$, $\u^{(2,k)}\in\mR^m$, and $\h^{(k)}\in\mR^{n}$ having iid standard normal elements, we set
  \begin{eqnarray}\label{eq:fl4}
\psi_{S,\infty}(f_{S},\calX,\calY,\p,\q,\c,x,y,s)  =
 \mE_{G,{\mathcal U}_{r+1}} \frac{1}{n\c_r} \log
\lp \mE_{{\mathcal U}_{r}} \lp \dots \lp \mE_{{\mathcal U}_3}\lp\lp\mE_{{\mathcal U}_2} \lp \lp Z_{S,\infty}\rp^{\c_2}\rp\rp^{\frac{\c_3}{\c_2}}\rp\rp^{\frac{\c_4}{\c_3}} \dots \rp^{\frac{\c_{r}}{\c_{r-1}}}\rp, \nonumber \\
 \end{eqnarray}
where
\begin{eqnarray}\label{eq:fl5}
Z_{S,\infty} & \triangleq & e^{D_{0,S,\infty}} \nonumber \\
 D_{0,S,\infty} & \triangleq  & \max_{\x\in\cX,\|\x\|_2=x} s \max_{\y\in\cY,\|\y\|_2=y}
 \lp \sqrt{n} f_{S}
+\sqrt{n}  y    \lp\sum_{k=2}^{r+1}c_k\h^{(k)}\rp^T\x
+ \sqrt{n} x \y^T\lp\sum_{k=2}^{r+1}b_k\u^{(2,k)}\rp \rp \nonumber  \\
 b_k & \triangleq & b_k(\p,\q)=\sqrt{\p_{k-1}-\p_k} \nonumber \\
c_k & \triangleq & c_k(\p,\q)=\sqrt{\q_{k-1}-\q_k}.
 \end{eqnarray}
The above technical details are sufficient to recall on the following theorem -- undoubtedly, one of the cornerstones of sfl RDT.
\begin{theorem} \cite{Stojnicflrdt23}
\label{thm:thmsflrdt1}  Consider large $n$ linear regime with  $\alpha\triangleq \lim_{n\rightarrow\infty} \frac{m}{n}$, remaining constant as  $n$ grows. Let $\cX\subseteq \mR^{n}$ and $\cY\subseteq \mR^m$ be two given sets and let the elements of  $A\in\mR^{m\times n}$
 be i.i.d. standard normals. Assume the complete sfl RDT frame from \cite{Stojnicsflgscompyx23} and consider a given function $f(\y):R^m\rightarrow R$. Set
\begin{align}\label{eq:thmsflrdt2eq1}
   \psi_{rp} & \triangleq - \max_{\x\in\cX} s \max_{\y\in\cY} \lp f(\y)+\y^TA\x \rp
   \qquad  \mbox{(\bl{\textbf{random primal}})} \nonumber \\
   \psi_{rd}(\p,\q,\c,x,y,s) & \triangleq    \frac{x^2y^2}{2}    \sum_{k=2}^{r+1}\Bigg(\Bigg.
   \p_{k-1}\q_{k-1}
   -\p_{k}\q_{k}
  \Bigg.\Bigg)
\c_k
  - \psi_{S,\infty}(f(\y),\calX,\calY,\p,\q,\c,x,y,s) \hspace{.03in} \mbox{(\bl{\textbf{fl random dual}})}. \nonumber \\
 \end{align}
Let $\hat{\p_0}\rightarrow 1$, $\hat{\q_0}\rightarrow 1$, and $\hat{\c_0}\rightarrow 1$, $\hat{\p}_{r+1}=\hat{\q}_{r+1}=\hat{\c}_{r+1}=0$, and let the non-fixed parts of $\hat{\p}\triangleq \hat{\p}(x,y)$, $\hat{\q}\triangleq \hat{\q}(x,y)$, and  $\hat{\c}\triangleq \hat{\c}(x,y)$ be the solutions of the following system
\begin{eqnarray}\label{eq:thmsflrdt2eq2}
   \frac{d \psi_{rd}(\p,\q,\c,x,y,s)}{d\p} =  0,\quad
   \frac{d \psi_{rd}(\p,\q,\c,x,y,s)}{d\q} =  0,\quad
   \frac{d \psi_{rd}(\p,\q,\c,x,y,s)}{d\c} =  0.
 \end{eqnarray}
 Then,
\begin{eqnarray}\label{eq:thmsflrdt2eq3}
    \lim_{n\rightarrow\infty} \frac{\mE_A  \psi_{rp}}{\sqrt{n}}
  & = &
\min_{x>0} \max_{y>0} \lim_{n\rightarrow\infty} \psi_{rd}(\hat{\p}(x,y),\hat{\q}(x,y),\hat{\c}(x,y),x,y,s) \qquad \mbox{(\bl{\textbf{strong sfl random duality}})},\nonumber \\
 \end{eqnarray}
where $\psi_{S,\infty}(\cdot)$ is as in (\ref{eq:fl4})-(\ref{eq:fl5}).
 \end{theorem}
\begin{proof}
Follows from the corresponding theorem proven in \cite{Stojnicflrdt23} in exactly the same way as Theorem 1 in \cite{Stojnicnegsphflrdt23}.
 \end{proof}

The above theorem holds for generic sets $\cX$ and $\cY$. The corollary below specializes it so that it is fully operational and ready to be utilized for studying ReLU injectivity capacity.
\begin{corollary}
\label{cor:cor1}  Assume the setup of Theorem \ref{thm:thmsflrdt1} with   $\cX=\mS^n$ and $\cY=\mS^m$, where $\mS^n$ and $\mS^m$ are the unit $n$ and $m$-dimensional Euclidean spheres. Set $f(\y)=-\y^T\z$ and
\begin{align}\label{eq:thmsflrdt2eq1a0}
   \psi_{rp} & \triangleq - \max_{\x\in\cX,\|\z\|_0<n} s \max_{\y\in\cY} \lp \y^TA\x  -\y^T\z \rp
   \qquad  \mbox{(\bl{\textbf{random primal}})} \nonumber \\
   \psi_{rd}(\p,\q,\c,1,1,s) & \triangleq    \frac{1}{2}    \sum_{k=2}^{r+1}\Bigg(\Bigg.
   \p_{k-1}\q_{k-1}
   -\p_{k}\q_{k}
  \Bigg.\Bigg)
\c_k
  - \psi_{S,\infty}(-\y^T\z,\calX,\calY,\p,\q,\c,1,1,s) \quad \mbox{(\bl{\textbf{fl random dual}})}. \nonumber \\
 \end{align}
Let the non-fixed parts of $\hat{\p}$, $\hat{\q}$, and  $\hat{\c}$ be the solutions of the following system
\begin{eqnarray}\label{eq:thmsflrdt2eq2a0}
   \frac{d \psi_{rd}(\p,\q,\c,1,1,s)}{d\p} =  0,\quad
   \frac{d \psi_{rd}(\p,\q,\c,1,1,s)}{d\q} =  0,\quad
   \frac{d \psi_{rd}(\p,\q,\c,1,1,s)}{d\c} =  0.
 \end{eqnarray}
 Then,
\begin{eqnarray}\label{eq:thmsflrdt2eq3a0}
    \lim_{n\rightarrow\infty} \frac{\mE_A  \psi_{rp}}{\sqrt{n}}
  & = &
 \lim_{n\rightarrow\infty} \psi_{rd}(\hat{\p},\hat{\q},\hat{\c},1,1,s) \qquad \mbox{(\bl{\textbf{strong sfl random duality}})},\nonumber \\
 \end{eqnarray}
where $\psi_{S,\infty}(\cdot)$ is as in (\ref{eq:fl4})-(\ref{eq:fl5}).
 \end{corollary}
\begin{proof}
Follows as a direct consequence of Theorem \ref{thm:thmsflrdt1}, after setting $f(\y)=-\y^T\z$, cosmetically adjusting for a maximization over $\z$, and noting that the specialized sets $\cX=\mS^n$ and $\cY=\mS^m$ have elements with unit Euclidean norms.
 \end{proof}

Trivial random primal concentrations allow for various  probabilistic variants of (\ref{eq:thmsflrdt2eq3}) and (\ref{eq:thmsflrdt2eq3a0}).
 As mentioned in \cite{Stojnicflrdt23}, these modifications are rather simple and bring no further insights. Instead of stating them, we proceed by working with the expected values.

\section{Practical utilization}
\label{sec:prac}

Theorem \ref{thm:thmsflrdt1} and Corollary \ref{cor:cor1} are conceptually very powerful. To have them reach their full potential and become of practical relevance,  all the associated quantities need be evaluated. As discussed on numerous occasions in \cite{Stojnicflrdt23,Stojnicnegsphflrdt23,Stojnicbinperflrdt23}, a couple of technical obstacles might be difficult to circumvent: \textbf{\emph{(i)}} It is not a priori clear what the right value for $r$ should be; \textbf{\emph{(ii)}} Sets $\cX$ and $\cY$ are of continuous type with no clear (typically desirable) component-wise structural characterizations rendering decoupling over $\x$ and $\y$ as not necessarily overly straightforward; and \textbf{\emph{(iii)}} Even if successful, handling of these obstacles may come at the expense of rather substantial and practically infeasible residual numerical evaluations. Throughout the rest of this section we address these problems and show that all potential obstacles can in fact be successfully surpassed.

As stated in  Corollary \ref{cor:cor1}, we take $\cX=\mS^n$ and $\cY=\mS^m$ and note that the key object for what follows is the so-called \emph{random dual}
\begin{align}\label{eq:prac1}
    \psi_{rd}(\p,\q,\c,1,1,s) & \triangleq    \frac{1}{2}    \sum_{k=2}^{r+1}\Bigg(\Bigg.
   \p_{k-1}\q_{k-1}
   -\p_{k}\q_{k}
  \Bigg.\Bigg)
\c_k
  - \psi_{S,\infty}(-\y^T\z,\cX,\cY,\p,\q,\c,1,1,s) \nonumber \\
  & =   \frac{1}{2}    \sum_{k=2}^{r+1}\Bigg(\Bigg.
   \p_{k-1}\q_{k-1}
   -\p_{k}\q_{k}
  \Bigg.\Bigg)
\c_k
  - \psi_{S,\infty}(-\y^T\z,\mS^n,\mS^m,\p,\q,\c,1,1,s) \nonumber \\
  & =   \frac{1}{2}    \sum_{k=2}^{r+1}\Bigg(\Bigg.
   \p_{k-1}\q_{k-1}
   -\p_{k}\q_{k}
  \Bigg.\Bigg)
\c_k
  - \frac{1}{n}\varphi(D^{(per)}(s),
  \c) - \frac{1}{n}\varphi(D^{(sph)}(s),\c), \nonumber \\
  \end{align}
where, based on (\ref{eq:fl4})-(\ref{eq:fl5}),
  \begin{eqnarray}\label{eq:prac2}
\varphi(D,\c) & \triangleq &
 \mE_{G,{\mathcal U}_{r+1}} \frac{1}{\c_r} \log
\lp \mE_{{\mathcal U}_{r}} \lp \dots \lp \mE_{{\mathcal U}_3}\lp\lp\mE_{{\mathcal U}_2} \lp
\lp    e^{D}   \rp^{\c_2}\rp\rp^{\frac{\c_3}{\c_2}}\rp\rp^{\frac{\c_4}{\c_3}} \dots \rp^{\frac{\c_{r}}{\c_{r-1}}}\rp, \nonumber \\
  \end{eqnarray}
and
\begin{eqnarray}\label{eq:prac3}
D^{(per)}(s) & = & \max_{\x\in\mS^n} \lp   s\sqrt{n}      \lp\sum_{k=2}^{r+1}c_k\h^{(k)}\rp^T\x  \rp \nonumber \\
  D^{(sph)}(s) & \triangleq  &  \max_{\|\z\|_0<n} s \max_{\y\in\mS^m}
\lp - \sqrt{n} \y^T\z + \sqrt{n}  \y^T\lp\sum_{k=2}^{r+1}b_k\u^{(2,k)}\rp \rp.
 \end{eqnarray}

\noindent  \underline{\red{\textbf{\emph{(i) Handling $D^{(per)}(-1)$:}}}}  A  simple evaluation first gives
\begin{eqnarray}\label{eq:prac4}
D^{(per)}(-1) & = & \max_{\x\in\mS^n}   \lp -\sqrt{n}      \lp\sum_{k=2}^{r+1}c_k\h^{(k)}\rp^T\x \rp =
\sqrt{n} \left\|       \sum_{k=2}^{r+1}c_k\h^{(k)} \right\|_2.
 \end{eqnarray}
We now note that the quantity on the right hand side of  (\ref{eq:prac4}) is identical to the quantity on the right hand side of equation (25) in \cite{Stojnicnegsphflrdt23}. One then proceeds as in \cite{Stojnicnegsphflrdt23},
utilizes the \emph{square root trick} (introduced on numerous occasions in \cite{StojnicLiftStrSec13,StojnicGardSphErr13,StojnicGardSphNeg13})  and ultimately obtains the following analogues to \cite{Stojnicnegsphflrdt23}'s equations (27) and (28)
\begin{eqnarray}\label{eq:prac4a01}
D^{(per)}(-1) & = &
   \min_{\gamma_{sq}^{(q)}} \lp \sum_{i=1}^nD^{(per)}_i(c_k) +\gamma_{sq}^{(q)}n \rp, \nonumber \\  .
 \end{eqnarray}
where
\begin{eqnarray}\label{eq:prac5}
D^{(per)}_i(c_k)=\frac{\lp\sum_{k=2}^{r+1}c_k\h_i^{(k)} \rp^2}{4\gamma_{sq}^{(q)}}.
\end{eqnarray}

\noindent \underline{\red{\textbf{\emph{(ii) Handling $D^{(sph)}(-1)$:}}}}  A simple observation gives
\begin{eqnarray}\label{eq:prac7}
  D^{(sph)}(-1) & \triangleq  & \hspace{-.1in} \max_{\|\z\|_0<n} - \max_{\y\in\mS^m}
\lp - \sqrt{n} \y^T\z + \sqrt{n}  \y^T\lp\sum_{k=2}^{r+1}b_k\u^{(2,k)}\rp \rp
=
\sqrt{n} \max_{\|\z\|_0<n} -
\left \| -  \z +  \lp\sum_{k=2}^{r+1}b_k\u^{(2,k)}\rp \right \|_2 \nonumber \\
& = &
-\sqrt{n} \min_{\|\z\|_0<n}
\left \| -  \z +  \lp\sum_{k=2}^{r+1}b_k\u^{(2,k)}\rp \right \|_2.
 \end{eqnarray}
Utilizing again the above mentioned square root trick from  \cite{StojnicLiftStrSec13,StojnicGardSphErr13,StojnicGardSphNeg13}, we further have
\begin{eqnarray}\label{eq:prac7a0}
  D^{(sph)}(-1)
  & =  &
- \sqrt{n} \min_{\|\z\|_0<n} \min_{\gamma_{sq}^{(p)}\geq 0}
\frac{\left \| -  \z +  \lp\sum_{k=2}^{r+1}b_k\u^{(2,k)}\rp \right \|_2^2}{4\gamma_{sq}^{(p)} \sqrt{n}} + \gamma_{sq}^{(p)} \sqrt{n}
\nonumber \\
& =  &
-  \min_{\gamma_{sq}^{(p)}\geq 0}
\lp \frac{ \min_{\|\z\|_0<n} \left \| -  \z +  \lp\sum_{k=2}^{r+1}b_k\u^{(2,k)}\rp \right \|_2^2}{4\gamma_{sq}^{(p)}} + \gamma_{sq}^{(p)} n \rp
\nonumber \\
& =  &
-  \min_{\gamma_{sq}^{(p)}\geq 0}
\lp \frac{\phi_z(\bar{\u})}{4\gamma_{sq}^{(p)}} + \gamma_{sq}^{(p)} n\rp,
 \end{eqnarray}
where we set
\begin{eqnarray}\label{eq:prac7a1}
  \bar{\u} \triangleq   \sum_{k=2}^{r+1}b_k\u^{(2,k)}\qquad \mbox{and} \qquad \phi_z(\bar{\u}) \triangleq  \min_{\|\z\|_0<n} \left \| -  \z + \bar{\u} \right \|_2^2.
 \end{eqnarray}
After observing
\begin{eqnarray}\label{eq:prac7a1a0}
\|\z\|_0=\sum_{i=1}^{m} \frac{\mbox{sign}(\z_i) + 1}{2},
\end{eqnarray}
and
\begin{eqnarray}\label{eq:prac7a1a1}
\|\z\|_0<n \qquad \Longleftrightarrow \qquad \sum_{i=1}^{m} \mbox{sign}(\z_i) < 2n-m,
\end{eqnarray}
one writes the Lagrangian
\begin{eqnarray}\label{eq:prac7a2}
  \phi_z(\bar{\u}) = \min_{\z} \max_{\nu\geq 0} {\mathcal L}(\z,\nu;\bar{\u}),
 \end{eqnarray}
where
\begin{eqnarray}\label{eq:prac7a3}
{\mathcal L}(\z,\nu;\bar{\u}) = \left \| -  \z + \bar{\u} \right \|_2^2  +\nu\|\z\|_0 -\nu n
= \|\bar{\u} \|_2^2  + \sum_{i=1}^{m} \lp - 2\bar{\u} \z_i + \z_i^2\rp +\nu \sum_{i=1}^{m} \mbox{sign}(\z_i) -\nu (2n -m),
 \end{eqnarray}
Setting
\begin{eqnarray}\label{eq:prac7a4}
 \phi_{z_i}(\z_i;\bar{\u}_i)  \triangleq    \lp - 2\bar{\u} \z_i + \z_i^2\rp +\nu \mbox{sign}(\z_i),
 \end{eqnarray}
and utilizing the strong Lagrangian duality, we also have
\begin{eqnarray}\label{eq:prac7a5}
  \phi_z(\bar{\u}) =  \max_{\nu\geq 0} \min_{\z} {\mathcal L}(\z,\nu;\bar{\u})
 =  \max_{\nu\geq 0}  \lp  \|\bar{\u} \|_2^2  +\nu(2n -m )  + \sum_{i=1}^{m}   \min_{\z_i}   \phi_{z_i}(\z_i;\bar{\u}_i) \rp.
 \end{eqnarray}
After solving the minimization over $\z$ we first obtain
\begin{eqnarray}\label{eq:prac7a6}
  \bar{\phi}_{z_i}(\bar{\u}_i) \triangleq \min_{\z_i}  \phi_{z_i}(\z_i;\bar{\u}_i)
  =
  \begin{cases}
    -\bar{\u}_i^2 -\nu , & \mbox{if } \bar{\u}_i\leq 0 \\
    \min(-\bar{\u}_i^2 +\nu,-\nu), &  \mbox{otherwise},
  \end{cases}
 \end{eqnarray}
and then also
\begin{eqnarray}\label{eq:prac7a7}
\bar{\phi}_{z_i}(\bar{\u}_i) =\min_{\z_i}  \phi_{z_i}(\z_i;\bar{\u}_i)
  =
  \begin{cases}
    -\bar{\u}_i^2 -\nu , & \mbox{if } \bar{\u}_i\leq 0 \\
    -\nu,  & \mbox{if } 0\leq \bar{\u}_i\leq \sqrt{2\nu} \\
    -\bar{\u}_i^2 +\nu, &  \mbox{otherwise}.
  \end{cases}
 \end{eqnarray}
A combination of (\ref{eq:prac7a0}), (\ref{eq:prac7a5}), and (\ref{eq:prac7a7}) gives
\begin{eqnarray}\label{eq:prac7a8}
  D^{(sph)}(-1)
 & =  &
-  \min_{\gamma_{sq}^{(p)}\geq 0} \max_{\nu\geq 0}
\lp \frac{  \|\bar{\u} \|_2^2  -\nu(2n -m )  + \sum_{i=1}^{m}   \min_{\z_i}   \phi_{z_i}(\z_i;\bar{\u}_i) }{4\gamma_{sq}^{(p)}} + \gamma_{sq}^{(p)} n\rp \nonumber \\
 & =  &
-  \min_{\gamma_{sq}^{(p)}\geq 0} \max_{\nu\geq 0}
\lp \frac{  \sum_{i=1}^{m}\bar{\u}_i^2  -\nu(2n -m )  + \sum_{i=1}^{m} \bar{\phi}_{z_i}(\bar{\u}_i) }{4\gamma_{sq}^{(p)}} + \gamma_{sq}^{(p)} n\rp \nonumber \\
 & =  &
-  \min_{\gamma_{sq}^{(p)}\geq 0} \max_{\nu\geq 0}
\lp   \sum_{i=1}^{m} D_i^{(sph)}(b_k) - \frac{    \nu(2n -m ) }{4\gamma_{sq}^{(p)}} + \gamma_{sq}^{(p)} n\rp,
 \end{eqnarray}
where
\begin{eqnarray}\label{eq:prac10}
   D_i^{(sph)}(b_k)
   =  \frac{   \bar{\u}_i^2   +  \bar{\phi}_{z_i}(\bar{\u}_i) }{4\gamma_{sq}^{(p)}}
   =  \frac{    \lp  \sum_{k=2}^{r+1}b_k\u_i^{(2,k)}
  \rp^2   +   \bar{\phi}_{z_i}( \sum_{k=2}^{r+1}b_k\u_i^{(2,k)}) }{4\gamma_{sq}^{(p)}}.
 \end{eqnarray}
Recalling on (\ref{eq:ex3}) and (\ref{eq:limlogpartfunsqrta0}) and utilizing concentrations, we then find
 \begin{eqnarray}
\mE_A \xi_{feas}(0)
  =
 \lim_{n\rightarrow\infty}\frac{\mE_A \max_{\x\in\mS^n}   \max_{\y\in\mS^m} \lp \y^TA\x -\y^T\z\rp}{\sqrt{n}}
 =
    \lim_{n\rightarrow\infty} \frac{\mE_A  \psi_{rp}}{\sqrt{n}}
   =
 \lim_{n\rightarrow\infty} \psi_{rd}(\hat{\p},\hat{\q},\hat{\c},1,1,-1).
  \label{eq:negprac11}
\end{eqnarray}
Keeping in mind (\ref{eq:prac1})-(\ref{eq:prac10}), the following theorem enables fitting the ReLU injectivity capacity characterization within the sfl RDT machinery.

\begin{theorem}
  \label{thm:negthmprac1}
  Assume the setup of Theorem \ref{thm:thmsflrdt1} and Corollary \ref{cor:cor1},
  and consider large $n$ linear regime with $\alpha=\lim_{n\rightarrow\infty} \frac{m}{n}$. Set
  \begin{eqnarray}\label{eq:thm2eq1}
\varphi(D,\c) & = &
 \mE_{G,{\mathcal U}_{r+1}} \frac{1}{\c_r} \log
\lp \mE_{{\mathcal U}_{r}} \lp \dots \lp \mE_{{\mathcal U}_3}\lp\lp\mE_{{\mathcal U}_2} \lp
\lp    e^{D}   \rp^{\c_2}\rp\rp^{\frac{\c_3}{\c_2}}\rp\rp^{\frac{\c_4}{\c_3}} \dots \rp^{\frac{\c_{r}}{\c_{r-1}}}\rp, \nonumber \\
  \end{eqnarray}
and
\begin{eqnarray}\label{eq:thm2eq2}
    \bar{\psi}_{rd}(\p,\q,\c,\gamma_{sq}^{(q)},\gamma_{sq}^{(p)},\nu) & \triangleq &      \frac{1}{2}    \sum_{k=2}^{r+1}\Bigg(\Bigg.
   \p_{k-1}\q_{k-1}
   -\p_{k}\q_{k}
  \Bigg.\Bigg)
\c_k
-\gamma_{sq}^{(q)} -\varphi(D_{1}^{(per)}(c_k),\c)
  \nonumber \\
& &
-\frac{\nu(2-\alpha)}{4\gamma_{sq}^{(p)} } + \gamma_{sq}^{(p)}  - \alpha\varphi(D_1^{(sph)}(b_k),\c),
  \end{eqnarray}
where $D_{i}^{(per)}(c_k)$ and $D_1^{(sph)}(b_k)$ are as in (\ref{eq:prac5}) and (\ref{eq:prac10}), respectively. Let the ``fixed'' parts of $\hat{\p}$, $\hat{\q}$, and $\hat{\c}$ satisfy $\hat{\p}_1\rightarrow 1$, $\hat{\q}_1\rightarrow 1$, $\hat{\c}_1\rightarrow 1$, $\hat{\p}_{r+1}=\hat{\q}_{r+1}=\hat{\c}_{r+1}=0$. Further, let the ``non-fixed'' parts of $\hat{\p}_k$, $\hat{\q}_k$, and $\hat{\c}_k$ ($k\in\{2,3,\dots,r\}$) and $\hat{\gamma}_{sq}^{(q)}$, $\hat{\gamma}_{sq}^{(p)}$, and $\hat{\nu}$ be the solutions of the following system of equations
  \begin{eqnarray}\label{eq:negthmprac1eq1}
   \frac{d \bar{\psi}_{rd}(\p,\q,\c,\gamma_{sq}^{(q)},\gamma_{sq}^{(p)},\nu)}{d\p} & =  & 0 \nonumber \\
   \frac{d \bar{\psi}_{rd}(\p,\q,\c,\gamma_{sq}^{(q)},\gamma_{sq}^{(p)},\nu)}{d\q} & =  & 0 \nonumber \\
   \frac{d \bar{\psi}_{rd}(\p,\q,\c,\gamma_{sq}^{(q)},\gamma_{sq}^{(p)},\nu)}{d\c} & =  & 0 \nonumber \\
   \frac{d \bar{\psi}_{rd}(\p,\q,\c,\gamma_{sq}^{(q)},\gamma_{sq}^{(p)},\nu)}{d\gamma_{sq}^{(q)}} & =  & 0\nonumber \\
   \frac{d \bar{\psi}_{rd}(\p,\q,\c,\gamma_{sq}^{(q)},\gamma_{sq}^{(p)},\nu)}{d\gamma_{sq}^{(p)}} & =  & 0\nonumber \\
   \frac{d \bar{\psi}_{rd}(\p,\q,\c,\gamma_{sq}^{(q)},\gamma_{sq}^{(p)},\nu)}{d\nu} & =  &  0.
 \end{eqnarray}
Consequently, let
\begin{eqnarray}\label{eq:prac17}
c_k(\hat{\p},\hat{\q})  & = & \sqrt{\hat{\q}_{k-1}-\hat{\q}_k} \nonumber \\
b_k(\hat{\p},\hat{\q})  & = & \sqrt{\hat{\p}_{k-1}-\hat{\p}_k}.
 \end{eqnarray}
 Then
 \begin{equation}
\mE_A\xi_{feas}(0)=   \bar{\psi}_{rd}(\hat{\p},\hat{\q},\hat{\c},\hat{\gamma}_{sq}^{(q)},\hat{\gamma}_{sq}^{(p)},\hat{\nu}).
  \label{eq:negthmprac1eq2}
\end{equation}
\end{theorem}
\begin{proof}
Follows through a combination of Theorem \ref{thm:thmsflrdt1}, Corollary \ref{cor:cor1}, and the sfl RDT machinery presented in \cite{Stojnicnflgscompyx23,Stojnicsflgscompyx23,Stojnicflrdt23,Stojnicnegsphflrdt23}.
\end{proof}

\subsection{Numerical evaluations}
\label{sec:nuemricalags}

To fully utilize the results of Theorem \ref{thm:negthmprac1}  one needs to conduct all the underlying numerical evaluations. We below present a systematic way of doing that. Starting with $r=1$, we proceed by incrementally increasing $r$ which  allows to follow the progress of the entire lifting mechanism in a neat way. Similarly to what was done in \cite{Stojnicnegsphflrdt23}, we
uncover the existence of a remarkable set of closed form analytical results that describe the parametric interconnections and substantially  facilitate the entire evaluation process.

\subsubsection{$r=1$ -- first level of lifting}
\label{sec:firstlev}

For the first level, we have $r=1$ and $\hat{\p}_1\rightarrow 1$ and $\hat{\q}_1\rightarrow 1$ which, coupled with $\hat{\p}_{r+1}=\hat{\p}_{2}=\hat{\q}_{r+1}=\hat{\q}_{2}=0$, and $\hat{\c}_{2}\rightarrow 0$, gives
\begin{align}\label{eq:negprac19}
    \bar{\psi}_{rd}^{(1)}    & =   \frac{1}{2}
\hat{\c}_2
   -\gamma_{sq}^{(q)}
     - \frac{1}{\hat{\c}_2}\log\lp \mE_{{\mathcal U}_2} e^{\hat{\c}_2 \frac{\lp  \sqrt{1-0}\h_1^{(2)} \rp^2}{4\gamma_{sq}^{(q)}}}\rp
\nonumber \\
  \nonumber \\
&  \quad +\gamma_{sq}^{(p)} -\frac{\nu(2-\alpha)}{4\gamma_{sq}^{(p)}}
- \alpha\frac{1}{\hat{\c}_2}\log\lp \mE_{{\mathcal U}_2} e^{-\hat{\c}_2     \frac{    \lp  \sqrt{1-0}\u_1^{(2,2)}
  \rp^2   +   \bar{\phi}_{z_1}( \sqrt{1-0}\u_1^{(2,2)}) }{4\gamma_{sq}^{(p)}}          }\rp \nonumber \\
& \rightarrow
  \frac{1}{2}
\hat{\c}_2
   -\gamma_{sq}^{(q)}
 -\frac{1}{\hat{\c}_2}\log\lp 1- \mE_{{\mathcal U}_2} \hat{\c}_2   \frac{\lp  \sqrt{1-0}\h_1^{(2)} \rp^2}{4\gamma_{sq}^{(q)}} \rp
     \nonumber \\
 &  \quad +\gamma_{sq}^{(p)} -\frac{\nu(2-\alpha)}{4\gamma_{sq}^{(p)}}
- \alpha\frac{1}{\hat{\c}_2}\log\lp 1 - \mE_{{\mathcal U}_2} \hat{\c}_2   \frac{    \lp  \sqrt{1-0}\u_1^{(2,2)}
  \rp^2   +   \bar{\phi}_{z_1}( \sqrt{1-0}\u_1^{(2,2)}) }{4\gamma_{sq}^{(p)}}         \rp \nonumber \\
\nonumber \\
& \rightarrow
  \frac{1}{2}
\hat{\c}_2
   -\gamma_{sq}^{(q)} -\frac{1}{4\gamma_{sq}^{(q)}}   +\gamma_{sq}^{(p)} -\frac{\nu(2-\alpha)}{4\gamma_{sq}^{(p)}}
+ \alpha \mE_{{\mathcal U}_2}   \frac{    \lp  \sqrt{1-0}\u_1^{(2,2)}
  \rp^2   +   \bar{\phi}_{z_1}( \sqrt{1-0}\u_1^{(2,2)}) }{4\gamma_{sq}^{(p)}},
\end{align}
  where we  write $ \bar{\psi}_{rd}^{(1)}   $ instead of  $ \bar{\psi}_{rd}^{(1)}(\hat{\p},\hat{\q},\hat{\c},\gamma_{sq}^{(q)},\gamma_{sq}^{(p)},\nu)  $ to simplify notation. Setting
\begin{eqnarray}\label{eq:negprac19a0a0a0}
  f_1(\nu)\triangleq  \alpha \mE_{{\mathcal U}_2} \lp \lp  \sqrt{1-0}\u_1^{(2,2)}
  \rp^2   +   \bar{\phi}_{z_1}( \sqrt{1-0}\u_1^{(2,2)})\rp,
\end{eqnarray}
and  utilizing
\begin{eqnarray}\label{eq:negprac19a0a0}
\frac{ \bar{\psi}_{rd}^{(1)}(\hat{\p},\hat{\q},\hat{\c},\gamma_{sq}^{(q)},\gamma_{sq}^{(p)},\nu) }{d\gamma_{sq}^{(q)}}=\frac{ \bar{\psi}_{rd}^{(1)}(\hat{\p},\hat{\q},\hat{\c},\gamma_{sq}^{(q)},\gamma_{sq}^{(p)},\nu) }{d\gamma_{sq}^{(p)}}=0,
  \end{eqnarray}
    we easily find
\begin{eqnarray}\label{eq:negprac19a0a1}
\hat{\gamma}_{sq}^{(q)} &  =& \frac{1}{2}\nonumber \\
\hat{\gamma}_{sq}^{(p)} & = & \frac{\sqrt{f_1(\nu)-\nu(2-\alpha)}}{2},
\end{eqnarray}
which, together with (\ref{eq:negprac19}), gives
\begin{eqnarray}\label{eq:negprac19a0}
    \bar{\psi}_{rd}^{(1)}   &  \rightarrow &
   -1 + \sqrt{f_1(\nu) -\nu(2-\alpha)}.
  \end{eqnarray}
 After  setting
 \begin{eqnarray}\label{eq:negprac19a1a0}
\bar{a} &  =  & \sqrt{2\nu},
  \end{eqnarray}
and
\begin{eqnarray}\label{eq:negprac19a1a1}
\bar{f}_x & = & -\lp\frac{e^{-a.^2/2} a}{\sqrt{2\pi}} + \frac{1}{2}\erfc\lp \frac{a}{\sqrt{2}}\rp   \rp \nonumber \\
\bar{f}_{21} & = & -\frac{1}{2}-\frac{\nu}{2}\nonumber \\
\bar{f}_{22} & = & \bar{f}_x + \frac{\nu}{2} \erfc\lp \frac{ \sqrt{2\nu}}{\sqrt{2}}  \rp\nonumber \\
\bar{f}_{23} & = & -\nu\lp\frac{1}{2}-\frac{1}{2} \erfc\lp \frac{ \sqrt{2\nu}}{\sqrt{2}}  \rp \rp \nonumber \\
\bar{f}_{2} & = & \bar{f}_{21} +\bar{f}_{22} + \bar{f}_{23},
 \end{eqnarray}
one computes the integrals and obtains
\begin{eqnarray}\label{eq:negprac19a1a2}
f_1(\nu)=\alpha(1+\bar{f}_{2}  ).
  \end{eqnarray}
From (\ref{eq:negprac19a0})  and  (\ref{eq:negprac19a1a2}), we finally have
\begin{eqnarray}\label{eq:negprac19a0b0}
    \bar{\psi}_{rd}^{(1)}    &  \rightarrow &
-1+ \sqrt{\alpha(1+\bar{f}_{2}  )-\nu(2-\alpha) },
  \end{eqnarray}
where $f_{21}$, $f_{22}$, and $f_{23}$ are given through  (\ref{eq:negprac19a1a0}), and (\ref{eq:negprac19a1a1}). One then finds  $\hat{\nu}$ as the solution of
\begin{eqnarray}\label{eq:dersystem0}
    \frac{d\bar{\psi}_{rd}^{(1)} }{d\nu}=0,
\end{eqnarray}
and for such $\hat{\nu}$ obtains the ReLU injectivity capacity as the $\alpha$ for which
\begin{eqnarray}\label{eq:dersystem0a0}
     \bar{\psi}_{rd}^{(1)} =0.
\end{eqnarray}
 In Table \ref{tab:tab1} we show the concrete values of all parameters.
\begin{table}[h]
\caption{$1$-sfl RDT parameters; ReLU injectivity capacity ;  $\hat{\c}_1\rightarrow 1$; $n,\beta\rightarrow\infty$; }\vspace{.1in}
\centering
\def\arraystretch{1.2}
\begin{tabular}{||l||c|c|c||c|c||c|c||c||c||}\hline\hline
 \hspace{-0in}$r$-sfl RDT                                             & $\hat{\gamma}_{sq}^{(q)}$   & $\hat{\gamma}_{sq}^{(p)}$      & $\hat{\nu}$ &  $\hat{\p}_2$ & $\hat{\p}_1$     & $\hat{\q}_2$  & $\hat{\q}_1$ &  $\hat{\c}_2$    & $\alpha_{ReLU}^{(inj,r)}$  \\ \hline\hline
$1$-sfl RDT                                       & $0.5$  & $0.5$   & $0.6304$ & $0$  & $\rightarrow 1$   & $0$ & $\rightarrow 1$
 &  $\rightarrow 0$  & \bl{$\mathbf{7.6477}$}
  \\ \hline\hline
  \end{tabular}
\label{tab:tab1}
\end{table}
The obtained capacity result precisely matches an upper-bounding one given in \cite{MBBDN23}. This is by a no surprise. The result in  \cite{MBBDN23} is obtained by utilizing the so-called \emph{plain} RDT that Stojnic created and introduced in a long series of works \cite{StojnicCSetam09,StojnicICASSP10var,StojnicGorEx10,StojnicRegRndDlt10,StojnicGardGen13}. As discussed in \cite{Stojnicflrdt23}, the plain RDT is a special case (obtained on the first level of lifting) of the fl RDT. Moreover, in  \cite{MBBDN23} it was shown that the replica methods assuming replica symmetry produce results that match those from Table \ref{tab:tab1}. This is also by a no surprise as two of the main postulates of the Stojnic's plain RDT machinery are: \textbf{\emph{1)}} The plain RDT matches the replica symmetric predictions; and \textbf{\emph{2)}} It further  makes the replica symmetric predictions as mathematically rigorous bounds. Stojnic's machinery \cite{StojnicCSetam09,StojnicICASSP10var,StojnicGorEx10,StojnicRegRndDlt10,StojnicGardGen13} goes further and establishes the conditional third postulate: \textbf{\emph{3)}} The plain RDT  produces not only bounding but also the \emph{exact} characterizations provided that the strong deterministic duality is in place -- also known as the ``\emph{strong deterministic duality implies strong random duality}'' principle. One of the consequences of these powerful results is that replica predictions are not only bounds but also the exact characterizations for convex problems. As here the underlying problems are noncovex, the strong deterministic duality is a priori not in place and one expects that the plain RDT characterizations are \emph{strict} bounds (i.e., that they are not the exact capacity values). As we will see below this is indeed the case.


\subsubsection{$r=2$ -- second level of lifting}
\label{sec:secondlev}

The setup presented in the previous subsection can be utilized for the second level as well. One now, however, has $r=2$ and, similarly to what happened on the first level, $\hat{\p}_1\rightarrow 1$ and $\hat{\q}_1\rightarrow 1$. On the other hand, $\hat{\c}_{2}\neq 0$. We below distinguish and address separately two scenarios: 1) partial   and 2) full second level of lifting.

\vspace{.1in}

\noindent \underline{\textbf{\emph{1) Partial second level of lifting:}}} In case of partial lifting we take $\p_2=\q_2=0$,
 which, together with $\hat{\p}_{r+1}=\hat{\p}_{3}=\hat{\q}_{r+1}=\hat{\q}_{3}=0$, allows to write, analogously to (\ref{eq:negprac19}),
\begin{align}\label{eq:parnegprac19}
    \bar{\psi}_{rd}^{(2,p)}    & =
     \frac{1}{2}
\c_2
   -\gamma_{sq}^{(q)}
     - \frac{1}{\c_2}\log\lp \mE_{{\mathcal U}_2} e^{\c_2 \frac{\lp  \h_1^{(2)} \rp^2}{4\gamma_{sq}^{(q)}}}\rp
+\gamma_{sq}^{(p)} -\frac{\nu(2-\alpha)}{4\gamma_{sq}^{(p)}}
- \alpha\frac{1}{\c_2}\log\lp \mE_{{\mathcal U}_2} e^{-\c_2     \frac{    \lp  \u_1^{(2,2)}
  \rp^2   +   \bar{\phi}_{z_1}( \u_1^{(2,2)}) }{4\gamma_{sq}^{(p)}}          }\rp \nonumber \\
& \rightarrow
     \frac{1}{2}
\c_2
   -\gamma_{sq}^{(q)}
     + \frac{1}{2\c_2}\log\lp \frac{2\gamma_{sq}^{(q)}-\c_2 }{2\gamma_{sq}^{(q)}}\rp
+\gamma_{sq}^{(p)} -\frac{\nu(2-\alpha)}{4\gamma_{sq}^{(p)}}
- \alpha\frac{1}{\c_2}\log\lp \mE_{{\mathcal U}_2} e^{-\c_2     \frac{    \lp  \u_1^{(2,2)}
  \rp^2   +   \bar{\phi}_{z_1}( \u_1^{(2,2)}) }{4\gamma_{sq}^{(p)}}          }\rp.
 \end{align}
From
\begin{eqnarray}\label{eq:parnegprac19a0}
\frac{    d\bar{\psi}_{rd}^{(2,p)} }{d\gamma_{sq}^{(q)}}=0,
 \end{eqnarray}
  one finds
\begin{eqnarray}\label{eq:parnegprac19a1}
 \hat{\gamma}_{sq}^{(q)}=\frac{\c_2+\sqrt{\c_2^2+4}}{4}.
  \end{eqnarray}
 Combining (\ref{eq:parnegprac19}) and (\ref{eq:parnegprac19a1}), we then have
 \begin{align}\label{eq:parnegprac19a2}
    \bar{\psi}_{rd}^{(2,p)}     & \rightarrow
     \frac{1}{2}
\c_2
   -\hat{\gamma}_{sq}^{(q)}
     + \frac{1}{2\c_2}\log\lp \frac{2\hat{\gamma}_{sq}^{(q)}-\c_2 }{2\hat{\gamma}_{sq}^{(q)}}\rp
+\gamma_{sq}^{(p)} -\frac{\nu(2-\alpha)}{4\gamma_{sq}^{(p)}}
- \alpha\frac{1}{\c_2}\log\lp \mE_{{\mathcal U}_2} e^{-\c_2     \frac{    \lp  \u_1^{(2,2)}
  \rp^2   +   \bar{\phi}_{z_1}( \u_1^{(2,2)}) }{4\gamma_{sq}^{(p)}}          }\rp.
 \end{align}
After setting
\begin{eqnarray}\label{eq:parnegprac19a3}
 \tilde{a} & = & \sqrt{2\nu} \nonumber \\
\tilde{f}_{21} & = & \frac{1}{2}e^{\frac{\c_2}{4\gamma_{sq}^{(p)}}\nu} \nonumber \\
\tilde{f}_{22} &  =& e^{\frac{\c_2}{4\gamma_{sq}^{(p)}}\nu} \frac{1}{2}\erfc\lp  \frac{\tilde{a}}{\sqrt{2}} \rp\nonumber \\
\tilde{a}_1 & = & \sqrt{2\nu}\sqrt{1+2\frac{\c_2}{4\gamma_{sq}^{(p)}}}\nonumber \\
\tilde{f}_{23} & = & \frac{ e^{\frac{\c_2}{4\gamma_{sq}^{(p)}}\nu} }{\sqrt{1+2\frac{\c_2}{4\gamma_{sq}^{(p)}}}}\lp\frac{1}{2}-\frac{1}{2}\erfc\lp  \frac{\tilde{a}}{\sqrt{2}} \rp \rp \nonumber \\
\tilde{f}_{2} & = & \tilde{f}_{21}+\tilde{f}_{22}+\tilde{f}_{23},
 \end{eqnarray}
and solving the remaining integrals one obtains
  \begin{align}\label{eq:parnegprac19a4}
    \bar{\psi}_{rd}^{(2,p)}     & \rightarrow
     \frac{1}{2}
\c_2
   -\hat{\gamma}_{sq}^{(q)}
     + \frac{1}{2\c_2}\log\lp \frac{2\hat{\gamma}_{sq}^{(q)}-\c_2 }{2\hat{\gamma}_{sq}^{(q)}}\rp
+\gamma_{sq}^{(p)} -\frac{\nu(2-\alpha)}{4\gamma_{sq}^{(p)}}
- \alpha\frac{1}{\c_2}\log\lp \tilde{f}_{2} \rp.
 \end{align}
One then finds $\hat{\gamma}_{sq}^{(p)}$  and  $\hat{\nu}$ as the solutions of
\begin{eqnarray}\label{eq:pardersystem0}
    \frac{d\bar{\psi}_{rd}^{(2,p)} }{d\gamma_{sq}^{(p)}}=    \frac{d\bar{\psi}_{rd}^{(2,p)} }{d\nu}=0,
\end{eqnarray}
and for such $\hat{\gamma}_{sq}^{(p)}$  and  $\hat{\nu}$ obtains the ReLU injectivity capacity as the $\alpha$ for which
\begin{eqnarray}\label{eq:pardersystem0a0}
     \bar{\psi}_{rd}^{(2,p)} =0.
\end{eqnarray}
The concrete optimal values of all parameters are shown  in Table \ref{tab:tab1p}.
\begin{table}[h]
\caption{$1$-sfl RDT parameters; ReLU injectivity capacity ;  $\hat{\c}_1\rightarrow 1$; $n,\beta\rightarrow\infty$; }\vspace{.1in}
\centering
\def\arraystretch{1.2}
\begin{tabular}{||l||c|c|c||c|c||c|c||c||c||}\hline\hline
 \hspace{-0in}$r$-sfl RDT                                             & $\hat{\gamma}_{sq}^{(q)}$   & $\hat{\gamma}_{sq}^{(p)}$      & $\hat{\nu}$ &  $\hat{\p}_2$ & $\hat{\p}_1$     & $\hat{\q}_2$  & $\hat{\q}_1$ &  $\hat{\c}_2$    & $\alpha_{ReLU}^{(inj,r)}$
  \\ \hline\hline
$1$-sfl RDT                                       & $0.5$  & $0.5$   & $0.6304$ & $0$  & $\rightarrow 1$   & $0$ & $\rightarrow 1$
 &  $\rightarrow 0$  & \bl{$\mathbf{7.6477}$}
  \\ \hline\hline
$2$-spl RDT                                       & $0.9412$  & $0.2656$   & $0.2785$ & $0$  & $\rightarrow 1$   & $0$ & $\rightarrow 1$
 &  $1.3513$  & \bl{$\mathbf{7.4486}$}
  \\ \hline\hline
  \end{tabular}
\label{tab:tab1p}
\end{table}
The obtained capacity result is what would be obtained utilizing the so-called \emph{partially lifted} RDT bounding strategy that Stojnic invented and introduced in, e.g., \cite{StojnicLiftStrSec13,StojnicGardSphErr13}. As discussed in \cite{Stojnicflrdt23}, partially lifted RDT is yet another special case of the fl RDT. Moreover, the authors of \cite{MBBDN23} were wondering what kind of result would partially lifted RDT produce when applied for characterization of the ReLU injectivity capacity. The results given in Table \ref{tab:tab1p} answer that question.

\vspace{.1in}

\noindent \underline{\textbf{\emph{2) Full second level of lifting:}}} For full second level of lifting we additionally have $\p_2\neq0$, and $\q_2\neq0$.
\begin{eqnarray}\label{eq:negprac24}
    \bar{\psi}_{rd}^{(2)}
    & = &
    \frac{1}{2}
(1-\p_2\q_2)\c_2     - \gamma_{sq}^{(q)}
 -\frac{1}{\c_2}\mE_{{\mathcal U}_3} \log\lp \mE_{{\mathcal U}_2} e^{\c_2\frac{\lp\sqrt{1-\q_2}\h_1^{(2)}+\sqrt{\q_2}\h_1^{(3)}\rp^2}{4\gamma_{sq}}}\rp \nonumber \\
& & +\gamma_{sq}^{(p)} -\frac{\nu(2-\alpha)}{4\gamma_{sq}^{(p)}}
- \alpha\frac{1}{\c_2}   \mE_{{\mathcal U}_3} \log\lp \mE_{{\mathcal U}_2} e^{-\c_2     \frac{    \lp  \sqrt{1-\p_2}\u_1^{(2,2)}
+ \sqrt{\p_2}\u_1^{(2,3)}  \rp^2   +   \bar{\phi}_{z_1}(  \sqrt{1-\p_2}\u_1^{(2,2)}
+ \sqrt{\p_2}\u_1^{(2,3)}   ) }{4\gamma_{sq}^{(p)}}          }   \rp \nonumber \\
& = &
  \frac{1}{2}
(1-\p_2\q_2)\c_2     - \gamma_{sq}^{(q)}
-  \Bigg(\Bigg. -\frac{1}{2\c_2} \log \lp \frac{2\gamma_{sq}^{(q)}-\c_2(1-\q_2)}{2\gamma_{sq}^{(q)}} \rp  +  \frac{\q_2}{2(2\gamma_{sq}^{(q)}-\c_2(1-\q_2))}   \Bigg.\Bigg)
\nonumber \\
 & &
 +\gamma_{sq}^{(p)} -\frac{\nu(2-\alpha)}{4\gamma_{sq}^{(p)}}
 - \frac{\alpha}{\c_2}\mE_{{\mathcal U}_3}\log\lp \mE_{{\mathcal U}_2}
{\cal Z}^{(2)}\rp,
     \end{eqnarray}
where
\begin{eqnarray}\label{eq:negprac24a0a0a0}
{\cal Z}^{(2)} & = &      e^{-\c_2     \frac{    \lp  \sqrt{1-\p_2}\u_1^{(2,2)}
+ \sqrt{\p_2}\u_1^{(2,3)}  \rp^2   +   \bar{\phi}_{z_1}(  \sqrt{1-\p_2}\u_1^{(2,2)}
+ \sqrt{\p_2}\u_1^{(2,3)}   ) }{4\gamma_{sq}^{(p)}}          }.
\end{eqnarray}
After  setting
 \begin{eqnarray}\label{eq:negprac24a0a0}
\bar{a}_0^{(2)} & = & -\u_1^{(2,3)}\frac{\sqrt{\p_2}}{\sqrt{1-\p_2}} \nonumber \\
 \hat{f}_{21}(x) & \triangleq &  e^{\frac{\c_2}{4\gamma_{sq}^{(p)}}\nu}
\lp \frac{1}{2} - \frac{1}{2}\erfc\lp  \frac{x}{\sqrt{2}} \rp \rp
\nonumber \\
\bar{a}_2^{(2)} & = &  \frac{\sqrt{2\nu} -\u_1^{(2,3)}\sqrt{\p_2}}  {\sqrt{1-\p_2}}\nonumber \\
 \hat{f}_{22}(x) & \triangleq &  e^{-\frac{\c_2}{4\gamma_{sq}^{(p)}}\nu} \frac{1}{2}\erfc\lp  \frac{x}{\sqrt{2}} \rp
\nonumber \\
C^{(2)} & = & \frac{\c_2}{4\gamma_{sq}^{(p)}} \nonumber \\
A^{(2)} & = & \sqrt{1-\p_2} \nonumber \\
B^{(2)} & = &  \u_1^{(2,3)}\sqrt{\p_2}\nonumber \\
D^{(2)} & = &  \bar{a}_0 \nonumber \\
F^{(2)} & = &  \bar{a}_2\nonumber \\
\cA & = & [A,B,C,D,F] \nonumber \\
\cA^{(2)} & = & [A^{(2)},B^{(2)},C^{(2)},D^{(2)},F^{(2)}]\nonumber \\
I_1(\cA)=I_1(A,B,C,D,F)& \triangleq &
\frac{
e^{  -\frac{B^2C}{2A^2C + 1}   }
  \lp \erf \lp \frac{(2AC (A F + B) + F) }  { \sqrt{4 A^2 C + 2} } \rp  -   \erf \lp \frac{(2AC (A D + B) + D) }  { \sqrt{4 A^2 C + 2} } \rp  \rp
 }
 {2\sqrt{2 A^2 C + 1}}\nonumber \\
\hat{f}_{23}(x) & \triangleq & e^{\frac{\c_2}{4\gamma_{sq}^{(\p_2)}}\nu}  x\nonumber \\
 \hat{f}_2^{(2)} & = & \hat{f}_{21}\lp\bar{a}_0^{(2)}\rp+\hat{f}_{22}\lp\bar{a}_2^{(2)}\rp+\hat{f}_{23}\lp   I_1  (  \cA^{(2)}   )    \rp,
\end{eqnarray}
and  solving the remaining integrals we obtain
\begin{eqnarray} \label{eq:negprac24a3a0}
  \mE_{{\mathcal U}_2}
{\cal Z}^{(2)} =    \hat{f}_2^{(2)}.
    \end{eqnarray}
A combination of (\ref{eq:negprac24})  and (\ref{eq:negprac24a3a0}) then gives
\begin{eqnarray}\label{eq:negprac25}
    \bar{\psi}_{rd}^{(2)}
 & = &
  \frac{1}{2}
(1-\p_2\q_2)\c_2     - \gamma_{sq}^{(q)}
-  \Bigg(\Bigg. -\frac{1}{2\c_2} \log \lp \frac{2\gamma_{sq}-\c_2(1-\q_2)}{2\gamma_{sq}} \rp  +  \frac{\q_2}{2(2\gamma_{sq}-\c_2(1-\q_2))}   \Bigg.\Bigg)
\nonumber \\
 & &
 +\gamma_{sq}^{(p)} -\frac{\nu(2-\alpha)}{4\gamma_{sq}^{(p)}}
 - \frac{\alpha}{\c_2}\mE_{{\mathcal U}_3}\log\lp \hat{f}_2^{(2)}
 \rp.
     \end{eqnarray}
 The above is then sufficient to form the following system of equations
\begin{eqnarray}\label{eq:dersystem1}
 \frac{d\bar{\psi}_{rd}^{(2)} }{d\p_2}
 =
 \frac{d\bar{\psi}_{rd}^{(2)} }{d\q_2}
 =
 \frac{d\bar{\psi}_{rd}^{(2)} }{d\c_2}
 =
 \frac{d\bar{\psi}_{rd}^{(2)} }{d\gamma_{sq}^{(q)}}
 =
 \frac{d\bar{\psi}_{rd}^{(2)} }{d\gamma_{sq}^{(p)}}
   =
   \frac{d\bar{\psi}_{rd}^{(2)} }{d\nu}=0.
\end{eqnarray}
Explicit expressions for the above derivatives are given in Appendix \ref{sec:appA}. After solving the system one obtains  $\hat{\p}_2$, $\hat{\q}_2$, $\hat{\c}_2$, $\hat{\gamma}_{sq}^{(q)}$, $\hat{\gamma}_{sq}^{(p)}$,  and $\hat{\nu}$.  The following remarkable closed form relations connect parameters' optimal values and provide a substantial help for numerical evaluations.

\begin{corollary}
  \label{cor:closedformrel1}
  Assume the setup of Theorem \ref{thm:negthmprac1} (and implicitly Theorem \ref{thm:thmsflrdt1} and Corollary \ref{cor:cor1}). Set $r=2$. Then
\begin{eqnarray}\label{eq:clform1}
\hat{\gamma}_{sq}^{(q)} & = & \frac{1}{2}\frac{1-\hat{\q}_2}{1-\hat{\p}_2}\sqrt{\frac{\hat{\p}_2}{\hat{\q}_2}} \nonumber \\
 \hat{\c}_2 & = &  \frac{1}{1-\hat{\p}_2} \sqrt{\frac{\hat{\p}_2}{\hat{\q}_2}} - \frac{1}{1-\hat{\q}_2} \sqrt{\frac{\hat{\q}_2}{\hat{\p}_2}}\nonumber \\
 \hat{\gamma}_{sq}^{(p)} & = & \frac{1}{4\hat{\gamma}_{sq}^{(q)}}.
\end{eqnarray}
\end{corollary}
\begin{proof}
The first two equalities  follow immediately from  equations (87) and (90) in \cite{Stojnicnegsphflrdt23} after recognizing that the derivation preceding these equations can be repeated here without any changes. The third equality follows after a combination of the first two and the derivative with respect to $\hat{\gamma}_{sq}^{(p)}$. The details are presented in Appendix \ref{sec:appB}.
\end{proof}

We present the concrete numerical values in Table \ref{tab:tab2}.  The results for the first level (1-sfl) from Table \ref{tab:tab1} and  for the partial second level (2-spf) from Table \ref{tab:tab1p} are shown in parallel as well.  One can then in a systematic way observe how the lifting mechanism is progressing.
\begin{table}[h]
\caption{$2$-sfl RDT parameters; ReLU injectivity capacity ;  $\hat{\c}_1\rightarrow 1$; $n,\beta\rightarrow\infty$;  }\vspace{.1in}
\centering
\def\arraystretch{1.2}
{\small
\begin{tabular}{||l||c|c|c||c|c||c|c||c||c||}\hline\hline
 \hspace{-0in}$r$-sfl RDT                                             & $\hat{\gamma}_{sq}^{(q)}$   & $\hat{\gamma}_{sq}^{(p)}$      & $\hat{\nu}$ &  $\hat{\p}_2$ & $\hat{\p}_1$     & $\hat{\q}_2$  & $\hat{\q}_1$ &  $\hat{\c}_2$    & $\alpha_{ReLU}^{(inj,r)}$
  \\ \hline\hline
$1$-sfl RDT                                       & $0.5$  & $0.5$   & $0.6304$ & $0$  & $\rightarrow 1$   & $0$ & $\rightarrow 1$
 &  $\rightarrow 0$  & \bl{$\mathbf{7.6477}$}
  \\ \hline\hline
$2$-spl RDT                                       & $0.9412$  & $0.2656$   & $0.2785$ & $0$  & $\rightarrow 1$   & $0$ & $\rightarrow 1$
 &  $1.3513$  & \bl{$\mathbf{7.4486}$}
  \\ \hline\hline
$2$-sfl RDT                                       & $3.6568$  & $0.0684$   & $0.0533$ & $0.7772$  & $\rightarrow 1$   & $0.1914$ & $\rightarrow 1$
 &  $8.4313$  & \bl{$\mathbf{6.7157}$}
  \\ \hline\hline
  \end{tabular}
  }
\label{tab:tab2}
\end{table}
As can be seen from the table, a solid improvement present already on the partial second level is substantially strengthened on the full second level.  Moreover, the capacity value $\approx 6.7157$  closely matches the one from \cite{MBBDN23} obtained utilizing statistical physics replica theory methods with one step of replica symmetry breaking.

\subsubsection{$r=3$ -- third level of lifting}
\label{sec:thirdlev}

For $r=3$, we have that $\hat{\p}_1\rightarrow 1$ and $\hat{\q}_1\rightarrow 1$  as well as  $\hat{\p}_{r+1}=\hat{\p}_{4}=\hat{\q}_{r+1}=\hat{\q}_{4}=0$. We first set
\begin{eqnarray}\label{eq:3negpraccalz}
{\cal Z}^{(3)} & = &      e^{-\c_2     \frac{    \lp  \sqrt{1-\p_2}\u_1^{(2,2)}
+ \sqrt{\p_2-\p_3}\u_1^{(2,3)}+ \sqrt{\p_3}\u_1^{(2,4)}  \rp^2   +   \bar{\phi}_{z_1}(  \sqrt{1-\p_2}\u_1^{(2,2)}
+ \sqrt{\p_2-\p_3}\u_1^{(2,3)}+ \sqrt{\p_3}\u_1^{(2,4)}  ) }{4\gamma_{sq}^{(p)}}          }.
\end{eqnarray}
wand write analogously to (\ref{eq:negprac24}),
\begin{align}\label{eq:3negprac24}
    \bar{\psi}_{rd}^{(3)}    & =   \frac{1}{2}
(1-\p_2\q_2)\c_2+ \frac{1}{2}
(\p_2\q_2-\p_3\q_3)\c_3
 \nonumber \\
& \quad -  \gamma_{sq}^{(q)} - \frac{1}{\c_3}\mE_{{\mathcal U}_4}\log\lp \mE_{{\mathcal U}_3} \lp \mE_{{\mathcal U}_2} e^{\c_2\frac{\lp\sqrt{1-\q_2}\h_1^{(2)} +\sqrt{\q_2-\q_3}\h_1^{(3)}+\sqrt{\q_3}\h_1^{(4)} \rp^2}{4 \gamma_{sq}^{(q)}}}\rp^{\frac{\c_3}{\c_2}}\rp
   \nonumber \\
& \quad + \gamma_{sq}^{(p)} - \frac{\nu(2-\alpha)}{\gamma_{sq}^{(p)}} - \frac{\alpha}{\c_3}\mE_{{\mathcal U}_4}\log  \lp \mE_{{\mathcal U}_3} \lp \mE_{{\mathcal U}_2}
{\cal Z}^{(3)}\rp^{\frac{\c_3}{\c_2}}  \rp
\nonumber \\
& =   \frac{1}{2}
(1-\p_2\q_2)\c_2+ \frac{1}{2}
(\p_2\q_2-\p_3\q_3)\c_3
 -\gamma_{sq}^{(q)}  - \Bigg(\Bigg. -\frac{1}{2\c_2} \log \lp \frac{2\gamma_{sq}^{(q)}-\c_2 (1-\q_2)}{2\gamma_{sq}^{(q)}} \rp
   \nonumber \\
& \quad
  -\frac{1}{2\c_3} \log \lp \frac{2\gamma_{sq}^{(q)}-\c_2 (1-\q_2)-\c_3 (\q_2-\q_3)}{2\gamma_{sq}^{(q)}-\q_2 (1-\q_2)} \rp  +  \frac{\q_3}{2(2\gamma_{sq}^{(q)}-\c_2 (1-\q_2)-\c_3(\q_2-\q_3))}   \Bigg.\Bigg)
     \nonumber \\
& \quad + \gamma_{sq}^{(p)} - \frac{\nu(2-\alpha)}{\gamma_{sq}^{(p)}} - \frac{\alpha}{\c_3}\mE_{{\mathcal U}_4}\log  \lp \mE_{{\mathcal U}_3} \lp \mE_{{\mathcal U}_2}
{\cal Z}^{(3)}\rp^{\frac{\c_3}{\c_2}}  \rp,
    \end{align}
where the fourth term on the right hand side of the first equality is handled as the corresponding quantity in \cite{Stojnicnegsphflrdt23}.  We set
\begin{eqnarray}\label{eq:3negprac24a0a0}
\bar{a}_0^{(3)} & = & -\frac{ \u_1^{(2,3)}\sqrt{\p_2-\p_3}  + \u_1^{(2,4)}\sqrt{\p_3} }{\sqrt{1-\p_2}} \nonumber \\
 \nonumber \\
\bar{a}_2^{(3)} & = &  \frac{\sqrt{2\nu} -  \lp \u_1^{(2,3)}\sqrt{\p_2-\p_3}  + \u_1^{(2,4)}\sqrt{\p_3} \rp   }  {\sqrt{1-\p_2}}\nonumber \\
 \nonumber \\
C^{(3)} & = & C^{(2)} \nonumber \\
A^{(3)} & = & A^{(2)} \nonumber \\
B^{(3)} & = &  \u_1^{(2,3)}\sqrt{\p_2-\p_3} +  \u_1^{(2,4)}\sqrt{\p_4}\nonumber \\
D^{(3)} & = &  \bar{a}_0^{(3)} \nonumber \\
F^{(3)} & = &  \bar{a}_2^{(3)}\nonumber \\
\cA^{(3)} & = & [A^{(3)},B^{(3)},C^{(3)},D^{(3)},F^{(3)}]\nonumber \\
\hat{f}_2^{(3)} & = & \hat{f}_{21}\lp\bar{a}_0^{(3)}\rp+\hat{f}_{22}\lp\bar{a}_2^{(3)}\rp+\hat{f}_{23}\lp I_1(\cA^{(3)})\rp,
 \end{eqnarray}
 where  functions  $\hat{f}_{21}(\cdot)$, $\hat{f}_{22}(\cdot)$, $\hat{f}_{23}(\cdot)$, and $I_1(\cdot)$ are as defined in (\ref{eq:negprac24a0a0}). After
  solving the remaining integrals one then obtains
\begin{eqnarray} \label{eq:3negprac24a3a0}
  \mE_{{\mathcal U}_2}
{\cal Z}^{(3)} =   \hat{f}_2^{(3)},
    \end{eqnarray}
and
\begin{eqnarray} \label{eq:3negprac24a3}
\mE_{{\mathcal U}_4}\log \lp  \mE_{{\mathcal U}_3} \lp \mE_{{\mathcal U}_2}
{\cal Z}^{(3)}\rp^{\frac{\c_3}{\c_2}} \rp
=
 \mE_{{\mathcal U}_4} \log \lp \mE_{{\mathcal U}_3} \lp  \hat{f}_2^{(3)}\rp^{\frac{\c_3}{\c_2}} \rp.
    \end{eqnarray}
A combination of (\ref{eq:3negprac24}), (\ref{eq:3negprac24a3a0}), and (\ref{eq:3negprac24a3}) gives
\begin{align}\label{eq:3negprac25}
    \bar{\psi}_{rd}^{(3)}
 & =   \frac{1}{2}
(1-\p_2\q_2)\c_2+ \frac{1}{2}
(\p_2\q_2-\p_3\q_3)\c_3
 -\gamma_{sq}^{(q)}  - \Bigg(\Bigg. -\frac{1}{2\c_2} \log \lp \frac{2\gamma_{sq}^{(q)}-\c_2 (1-\q_2)}{2\gamma_{sq}^{(q)}} \rp
   \nonumber \\
& \quad
  -\frac{1}{2\c_3} \log \lp \frac{2\gamma_{sq}^{(q)}-\c_2 (1-\q_2)-\c_3 (\q_2-\q_3)}{2\gamma_{sq}^{(q)}-\q_2 (1-\q_2)} \rp  +  \frac{\q_3}{2(2\gamma_{sq}^{(q)}-\c_2 (1-\q_2)-\c_3(\q_2-\q_3))}   \Bigg.\Bigg)
     \nonumber \\
& \quad + \gamma_{sq}^{(p)} - \frac{\nu(2-\alpha)}{\gamma_{sq}^{(p)}} - \frac{\alpha}{\c_3}\mE_{{\mathcal U}_4}\log  \lp \mE_{{\mathcal U}_3} \lp \hat{f}_2^{(3)}\rp^{\frac{\c_3}{\c_2}}  \rp,
    \end{align}
with $\hat{f}_2^{(3)}$ as in (\ref{eq:3negprac24a0a0}). Proceeding as in the previous sections, one utilizes the above to form the following system of equations
\begin{eqnarray}\label{eq:dersystem1}
 \frac{d\bar{\psi}_{rd}^{(3)} }{d\p_2}
 =
 \frac{d\bar{\psi}_{rd}^{(3)} }{d\q_2}
 =
 \frac{d\bar{\psi}_{rd}^{(3)} }{d\c_2}
 =
 \frac{d\bar{\psi}_{rd}^{(3)} }{d\p_3}
 =
 \frac{d\bar{\psi}_{rd}^{(3)} }{d\q_3}
 =
 \frac{d\bar{\psi}_{rd}^{(3)} }{d\c_3}
 =
 \frac{d\bar{\psi}_{rd}^{(3)} }{d\gamma_{sq}^{(q)}}
 =
 \frac{d\bar{\psi}_{rd}^{(3)} }{d\gamma_{sq}^{(p)}}
   =
   \frac{d\bar{\psi}_{rd}^{(3)} }{d\nu}=0,
\end{eqnarray}
and obtain as the solutions of the system  $\hat{\p}_2$, $\hat{\q}_2$, $\hat{\c}_2$, $\hat{\p}_3$, $\hat{\q}_3$, $\hat{\c}_3$, $\hat{\gamma}_{sq}^{(q)}$, $\hat{\gamma}_{sq}^{(p)}$,  and $\hat{\nu}$. Explicit expressions for the above derivatives are given in Appendix \ref{sec:appC}. We again discover that he optimal values of the parameters are interconnected via explicit closed form relations similar to the ones obtained on the second level. The results are summarized in the following corollary.
\begin{corollary}
  \label{cor:3closedformrel1}
  Assume the setup of Theorem \ref{thm:negthmprac1} (and implicitly Theorem \ref{thm:thmsflrdt1} and Corollary \ref{cor:cor1}). Set $r=3$. Then
\begin{eqnarray}\label{eq:3clform1}
\hat{\gamma}_{sq} & = & \frac{1}{2}\frac{1-\hat{\q}_2}{1-\hat{\p}_2}\frac{\hat{\p}_2-\hat{\p}_3}{\hat{\q}_2-\hat{\q}_3}\sqrt{\frac{\hat{\q}_3}{\hat{\p}_3}}  \nonumber \\
 \hat{\c}_2 & = &  \lp
 \frac{1}{1-\hat{\p}_2} \frac{\hat{\p}_2-\hat{\p}_3}{\hat{\q}_2-\hat{\q}_3} \sqrt{\frac{\hat{\q}_3}{\hat{\p}_3}}
 -
 \frac{1}{1-\hat{\q}_2} \frac{\hat{\q}_2-\hat{\q}_3}{\hat{\p}_2-\hat{\p}_3} \sqrt{\frac{\hat{\p}_3}{\hat{\q}_3}}
    \rp  \nonumber \\
 \hat{\c}_3 & = &   \lp
\frac{1}{\hat{\p}_2-\hat{\p}_3} \sqrt{\frac{\hat{\p}_3}{\hat{\q}_3}}
-
\frac{1}{\hat{\q}_2-\hat{\q}_3} \sqrt{\frac{\hat{\q}_3}{\hat{\p}_3}}
    \rp   \nonumber \\
 \gamma_{sq}^{(p)} & = & \frac{1}{4\gamma_{sq}^{(q)}}.
\end{eqnarray}
\end{corollary}
\begin{proof}
The first three equalities  follow immediately as analogues to (142), (145), and (147)  in \cite{Stojnicnegsphflrdt23} after recognizing that the derivation preceding these equations can be repeated here without any changes. The fourth equality follows after a combination of the first three and the derivative with respect to $\hat{\gamma}_{sq}^{(p)}$. The details are presented in Appendix \ref{sec:appD}.
\end{proof}

The concrete numerical values obtained based on the above are shown in Table \ref{tab:tab3}. Results from  Tables \ref{tab:tab1}, \ref{tab:tab1p}, and \ref{tab:tab2} are shown in parallel as well so that one can convenienntly follow the progress of the lifting mechanism.
\begin{table}[h]
\caption{$3$-sfl RDT parameters; ReLU injectivity capacity ;  $\hat{\c}_1\rightarrow 1$; $n,\beta\rightarrow\infty$. }\vspace{.1in}
\centering
\def\arraystretch{1.2}
{\footnotesize
\begin{tabular}{||l||c|c|c||c|c|c||c|c|c||c|c||c||}\hline\hline
 \hspace{-0in}$r$-sfl                                             & $\hat{\gamma}_{sq}$   & $\hat{\gamma}$      & $\hat{\nu}$ &  $\hat{\p}_3$ &  $\hat{\p}_2$ & $\hat{\p}_1$     & $\hat{\q}_3$  & $\hat{\q}_2$  & $\hat{\q}_1$ &  $\hat{\c}_3$  &  $\hat{\c}_2$    & $\alpha_{ReLU}^{(inj,r)}$  \\ \hline\hline
$1$-sfl                                       & $0.5$  & $0.5$   & $0.6304$ & $0$ & $0$  & $\rightarrow 1$ & $0$   & $0$ & $\rightarrow 1$
& $0$  &  $\rightarrow 0$  & \bl{$\mathbf{7.6477}$}
  \\ \hline\hline
$2$-spl  \hspace{-.05in}                                      & $0.9412$  & $0.2656$   & $0.2785$ & $0$  & $0$  & $\rightarrow 1$ & $0$    & $0$ & $\rightarrow 1$
& $0$  &  $1.3513$  & \bl{$\mathbf{7.4486}$}
  \\ \hline\hline
$2$-sfl                                         & $3.6568$  & $0.0684$   & $0.0533$ & $0$& $0.7772$  & $\rightarrow 1$ & $0$  & $0.1914$ & $\rightarrow 1$
& $0$ &  $8.4313$  & \bl{$\mathbf{6.7157}$}
  \\ \hline\hline
$3$-sfl                                        & $5.2521$  & $0.0476$   & $0.0357$ & $0.7411$  & $0.9766$  & $\rightarrow 1$  & $0.1672$   & $0.4279$ & $\rightarrow 1$
& $7.1182$  &  $14.2862$  & \bl{$\mathbf{6.7004}$}
  \\ \hline\hline
  \end{tabular}
  }
\label{tab:tab3}
\end{table}
As can be seem from the table, not much of an improvement is observed  compared to the full second level. Given that the relative improvement is $\sim 0.1\%$, the convergence of the lifting mechanism is remarkably fast effectively rendering further levels of lifting for all practical purposes unnecessary. In fact, the third level results shown in Table \ref{tab:tab3} are also fairly close to $\approx 6.698$ which was obtained in \cite{MBBDN23} through the replica predictions  assuming infinite steps of symmetry breaking.

\section{Conclusion}
\label{sec:conc}

We studied the injectivity capacity of a ReLU network layer. This capacity is conceptually a different quantity from the classical spherical perceptron capacities. After introducing the so-called $\ell_0$ spherical perceptron, we established that it is isomorphic to the ReLU injectivity. As $\ell_0$ spherical perceptron is a random feasibility problem (rfp),
we then utilized the  \emph{fully lifted random duality theory} (fl RDT) to create a generic powerful program for handling  such problems. Due to the underlying isomorphism, this was then sufficient to implicitly handle the ReLU injectivity as well. The fl RDT machinery is put to  practical use after a sizeable set of numerical evaluations is conducted. The lifting mechanism is observed to converge  very fast with estimated quantities relative corrections not exceeding $\sim 0.1\%$ already on the third level of lifting.

The concrete capacity values on the first level of lifting match the ones that would be obtained through the \emph{plain} RDT developed in \cite{StojnicCSetam09,StojnicICASSP10var,StojnicGorEx10,StojnicRegRndDlt10,StojnicGardGen13}. As such they also match the predictions obtained  in \cite{MBBDN23} based on replica methods assuming replica symmetry. On the other hand, the concrete capacity values obtained on the partial second level of lifting match the ones that would be obtained through the \emph{partially lifted} RDT invented and introduced in, e.g., \cite{StojnicLiftStrSec13,StojnicGardSphErr13}. The results obtained on the full second level of lifting identically match the replica methods predictions obtained  in \cite{MBBDN23} assuming one step of symmetry breaking. Moreover, the results obtained on the third level are fairly close to the replica predictions given in \cite{MBBDN23} assuming infinite steps of symmetry breaking.

We also uncover a set of remarkable closed form explicit analytical relations among key lifting parameters. They turned out to be of incredible importance in handling all the required numerical work. At the same time they also shed a new light on  beautiful parametric interconnections that lifting structure possesses.

As the developed methodologies are very generic  many extensions and generalizations are possible. Among others, they include studying stability and Lipshitz constants, multi-layered nets, different models, including planted and/or noise corrupted ones and so on. Given that the associated technicalities are problem specific, we leave their detailed discussions for separate papers.

\begin{singlespace}
\bibliographystyle{plain}
\bibliography{nflgscompyxRefs}
\end{singlespace}

\appendix

\section{Derivatives -- Second level of lifting}
\label{sec:appA}

We recall on (\ref{eq:negprac25})
\begin{eqnarray}\label{eq:app1negprac25}
    \bar{\psi}_{rd}^{(2)}
 & = &
  \frac{1}{2}
(1-\p_2\q_2)\c_2     - \gamma_{sq}^{(q)}
-  \Bigg(\Bigg. -\frac{1}{2\c_2} \log \lp \frac{2\gamma_{sq}^{(q)}-\c_2(1-\q_2)}{2\gamma_{sq}^{(q)}} \rp  +  \frac{\q_2}{2(2\gamma_{sq}^{(q)}-\c_2(1-\q_2))}   \Bigg.\Bigg)
\nonumber \\
 & &
 +\gamma_{sq}^{(p)} -\frac{\nu(2-\alpha)}{4\gamma_{sq}^{(p)}}
 - \frac{\alpha}{\c_2}\mE_{{\mathcal U}_3}\log\lp \hat{f}_2^{(2)}
 \rp.
     \end{eqnarray}
After  setting
\begin{eqnarray}\label{eq:app1eq1}
I_R^{(2)} & = &
\gamma_{sq}^{(q)}  + \Bigg(\Bigg. -\frac{1}{2\c_2} \log \lp \frac{2\gamma_{sq}^{(q)}-\c_2(1-\q_2)}{2\gamma_{sq}^{(q)}} \rp  +  \frac{\q_2}{2(2\gamma_{sq}^{(q)}-\c_2(1-\q_2))}   \Bigg.\Bigg)
\nonumber \\
I_L^{(2)} & = &
\frac{\alpha}{\c_2}\mE_{{\mathcal U}_3}\log\lp \hat{f}_2^{(2)}
 \rp,
     \end{eqnarray}
one can trivially rewrite (\ref{eq:app1negprac25})
\begin{eqnarray}\label{eq:app1eq1a0}
    \bar{\psi}_{rd}^{(2)}
 & = &
  \frac{1}{2}
(1-\p_2\q_2)\c_2     - I_R^{(2)}
  +\gamma_{sq}^{(p)} -\frac{\nu(2-\alpha)}{4\gamma_{sq}^{(p)}}
 -  I_L^{(2)}.
     \end{eqnarray}

\subsection{$\c_2$ derivative}

 We first find
\begin{eqnarray}\label{eq:app1eq2}
\frac{dI_R^{(2)} }{d\c_2}& = &
\frac{1}{2\c_2^2} \log \lp \frac{2\gamma_{sq}^{(q)}-\c_2(1-\q_2)}{2\gamma_{sq}^{(q)}} \rp
 + \frac{1-\q_2}{2\c_2(2\gamma_{sq}^{(q)}-\c_2(1-\q_2))}
 + \frac{\q_2(1-\q_2)}{2(2\gamma_{sq}^{(q)}-\c_2(1-\q_2))^2},
     \end{eqnarray}
and recalling on (\ref{eq:negprac24a0a0}) further write
 \begin{eqnarray}\label{eq:app1eq3}
  \frac{d\hat{f}_{21}\lp \bar{a}_0^{(2)}\rp } {d\c_2} &  = & \frac{\nu}{4\gamma_{sq}^{(p)}} e^{\frac{\c_2}{4\gamma_{sq}^{(p)}}\nu}
\lp \frac{1}{2} - \frac{1}{2}\erfc\lp  \frac{ \bar{a}_0^{(2)}}{\sqrt{2}} \rp \rp
\nonumber \\
 \frac{d\hat{f}_{22}\lp \bar{a}_2^{(2)}\rp } {d\c_2} & = &  -\frac{\nu}{4\gamma_{sq}^{(p)}}  e^{-\frac{\c_2}{4\gamma_{sq}^{(p)}}\nu} \frac{1}{2}\erfc\lp  \frac{\bar{a}_2^{(2)}}{\sqrt{2}} \rp
 \nonumber \\
 \frac{dC^{(2)}} {d\c_2} & = &  \frac{1}{4\gamma_{sq}^{(p)}}.
 \end{eqnarray}
We now set
 \begin{eqnarray}\label{eq:app1eq4a0}
\cA = [A,B,C,D,F]
\end{eqnarray}
and
 \begin{eqnarray}\label{eq:app1eq4}
L_4(\cA) & = & 2\sqrt{2 A^2 C + 1}
\nonumber\\
L_1(\cA) & = & e^{-\frac{4B^2C}{L_4^2}}
\nonumber\\
L_5(\cA) & = & \sqrt{2}(2AC (AF + B) + F)
\nonumber\\
L_6(\cA) & = & \sqrt{2}(2 AC(A D + B) + D)
\nonumber\\
L_2(\cA) & = & \erf\lp \frac{L_5}{L_4}\rp
\nonumber\\
L_3(\cA) & = & \erf\lp \frac{L_6}{L_4}\rp
\nonumber \\
 I_1(\cA) & =  & \frac{L_1(\cA) (L_2(\cA) - L_3(\cA))}{L_4(\cA)}.
\end{eqnarray}
Assuming that $\cA$ depends on $\c_2$ only via $C$, we also find
 \begin{eqnarray}\label{eq:app1eq5}
\frac{dL_4(\cA)}{d\c_2} & = & \frac{2A^2}{\sqrt{2 A^2C + 1}}  \frac{dC} {d\c_2}
\nonumber\\
\frac{dL_1(\cA)}{d\c_2} & = & \lp-\frac{4B^2}{L_4(\cA)^2}\frac{dC} {d\c_2} + \frac{8B^2C}{L_4(\cA)^3} \frac{dL_4(\cA)}{d\c_2} \rp L_1(\cA)
\nonumber\\
\frac{dL_5(\cA)}{d\c_2} & = & \sqrt{2}\lp2A\frac{dC} {d\c_2} (AF + B) \rp
\nonumber\\
\frac{dL_6(\cA)}{d\c_2} & = & \sqrt{2}\lp2A\frac{dC} {d\c_2} (AD + B) \rp
\nonumber\\
\frac{dL_2(\cA)}{d\c_2} & = & \frac{2}{\sqrt{\pi}} e^{-\lp \frac{L_5(\cA)}{L_4(\cA)}\rp^2}  \lp\frac{dL_5(\cA)}{d\c_2}\frac{1}{L_4(\cA)}- \frac{L_5(\cA)}{L_4(\cA)^2} \frac{dL_4(\cA)}{d\c_2} \rp
\nonumber\\
\frac{dL_3(\cA)}{d\c_2} & = & \frac{2}{\sqrt{\pi}} e^{-\lp \frac{L_6(\cA)}{L_4(\cA)}\rp^2}  \lp\frac{dL_6(\cA)}{d\c_2}\frac{1}{L_4(\cA)}- \frac{L_6(\cA)}{L_4(\cA)^2} \frac{dL_4(\cA)}{d\c_2} \rp
\nonumber \\
 \frac{dI_1(\cA )}{d\c_2} & = &
\frac{dL_1(\cA)}{d\c_2}\frac{(L_2(\cA)-L_3(\cA))}{L_4(\cA)} + \frac{L_1(\cA)}{L_4(\cA)}\lp\frac{dL_2(\cA)}{d\c_2}-\frac{dL_3(\cA)}{d\c_2}\rp
\nonumber \\
& & - \frac{L_1(\cA)(L_2(\cA)-L_3(\cA))}{L_4(\cA)^2}\frac{dL_4(\cA)}{d\c_2}
\nonumber \\
  \frac{d\hat{f}_{23}\lp I_1(\cA^{(2)})\rp } {d\c_2}
  & = &
\frac{\nu}{4\gamma_{sq}^{(p)}} e^{\frac{\c_2}{4\gamma_{sq}^{(p)}}\nu} I_1(\cA^{(2)}) + e^{\frac{\c_2}{4\gamma_{sq}^{(p)}}\nu}  \frac{dI_1(\cA^{(2)} )}{d\c_2}.
\end{eqnarray}
Taking
 \begin{eqnarray}\label{eq:app1eq4a0a0}
\cA^{(i)} = [A^{(i)},B^{(i)},C^{(i)},D^{(i)},F^{(i)}], i=2,
\end{eqnarray}
and observing that $\cA^{(2)}$ indeed depends on $\c_2$ only via $C^{(2)}$, we can after solving the remaining integrals write
through a combination of (\ref{eq:negprac24a0a0}), (\ref{eq:app1eq3}), and (\ref{eq:app1eq4a0a0})
 \begin{eqnarray}\label{eq:app1eq7}
  \frac{d\hat{f}_2^{(2)}} {d\c_2}  =
  \frac{d\hat{f}_{21}\lp \bar{a}_0^{(2)}\rp } {d\c_2}
  +   \frac{d\hat{f}_{22}\lp \bar{a}_0^{(2)}\rp } {d\c_2}
   +    \frac{d\hat{f}_{23}\lp I_1(\cA^{(2)})\rp } {d\c_2}.
\end{eqnarray}
From (\ref{eq:app1eq1}), one then easily finds
\begin{eqnarray}\label{eq:app1eq8}
 \frac{dI_L^{(2)}}{d\c_2} & = &
-\frac{\alpha}{\c_2^2}\mE_{{\mathcal U}_3} \log\lp  \hat{f}_2^{(2)}
 \rp
 +
\frac{\alpha}{\c_2}\mE_{{\mathcal U}_3} \lp \frac{1}{\hat{f}_2^{(2)}}  \frac{d\hat{f}_2^{(2)}} {d\c_2}
 \rp.
     \end{eqnarray}
A combination of (\ref{eq:app1eq1}), (\ref{eq:app1eq1a0}),,(\ref{eq:app1eq2}), (\ref{eq:app1eq7}), and (\ref{eq:app1eq8})  then gives
\begin{eqnarray}\label{eq:app1eq9}
  \frac{d  \bar{\psi}_{rd}^{(2)}}{d\c_2}
 & = &
  \frac{1}{2}
(1-\p_2\q_2)
\nonumber \\
& &
-
\lp \frac{1}{2\c_2^2} \log \lp \frac{2\gamma_{sq}^{(q)}-\c_2(1-\q_2)}{2\gamma_{sq}^{(q)}} \rp
 + \frac{1-\q_2}{2\c_2(2\gamma_{sq}^{(q)}-\c_2(1-\q_2))}
 + \frac{\q_2(1-\q_2)}{2(2\gamma_{sq}^{(q)}-\c_2(1-\q_2))^2} \rp
\nonumber \\
 & &
  - \lp   -\frac{\alpha}{\c_2^2}\mE_{{\mathcal U}_3} \log\lp  \hat{f}_2^{(2)}
 \rp
 +
\frac{\alpha}{\c_2}\mE_{{\mathcal U}_3} \lp \frac{1}{\hat{f}_2^{(2)}}  \frac{d\hat{f}_2^{(2)}} {d\c_2}
 \rp  \rp.
     \end{eqnarray}

\subsection{$\p_2$ derivative}

We trivially observe
\begin{eqnarray}\label{eq:papp1eq2}
\frac{dI_R^{(2)} }{d\p_2}& = &
 0,
     \end{eqnarray}
and recalling on (\ref{eq:negprac24a0a0})  write
 \begin{eqnarray}\label{eq:papp1eq3}
  \frac{d\bar{a}_{0}^{(2)} } {d\p_2} &  = & -\u_1^{(2,3)}/2/\sqrt{\p_2}/\sqrt{1-\p_2}
  - \u_1^{(2,3)}\sqrt{\p_2}/2/\sqrt{1-\p_2}^3
\nonumber \\
\frac{d\bar{a}_{2}^{(2)} } {d\p_2} &  = & 1/2\sqrt{2\nu}/\sqrt{1-\p_2}^3
-\u_1^{(2,3)}/2/\sqrt{\p_2}/\sqrt{1-\p_2}
  - \u_1^{(2,3)}\sqrt{\p_2}/2/\sqrt{1-\p_2}^3
\nonumber \\
  \frac{d\hat{f}_{21}\lp \bar{a}_0^{(2)}\rp } {d\p_2} &  = &  e^{\frac{\c_2}{4\gamma_{sq}^{(p)}}\nu}
\frac{1}{\sqrt{\pi}} e^{-\lp \frac{ \bar{a}_0^{(2)}}{\sqrt{2}}   \rp^2}  \frac{1}{\sqrt{2}}  \frac{d\bar{a}_0^{(2)} }{d\p_2}
\nonumber \\
 \frac{d\hat{f}_{22}\lp \bar{a}_2^{(2)}\rp } {d\p_2} & = & - e^{-\frac{\c_2}{4\gamma_{sq}^{(p)}}\nu}
\frac{1}{\sqrt{\pi}} e^{-\lp \frac{ \bar{a}_2^{(2)}}{\sqrt{2}}   \rp^2}  \frac{1}{\sqrt{2}}  \frac{d\bar{a}_2^{(2)} }{d\p_2}   \nonumber \\
 \frac{dC^{(2)}} {d\p_2} & = &  0
 \nonumber \\
 \frac{dA^{(2)}} {d\p_2} & = &  -1/2/\sqrt{1-\p_2}
 \nonumber \\
 \frac{dB^{(2)}} {d\p_2} & = &  \u_1^{(2,3)}/2/\sqrt{\p_2}
 \nonumber \\
 \frac{dD^{(2)}} {d\p_2} & = &    \frac{d\bar{a}_{0}^{(2)} } {d\p_2}
 \nonumber \\
 \frac{dF^{(2)}} {d\p_2} & = &    \frac{d\bar{a}_{2}^{(2)} } {d\p_2}.
  \end{eqnarray}
Recalling on (\ref{eq:app1eq4a0}) and (\ref{eq:app1eq4}) and  assuming that $\cA$ depends on $\p_2$ only via $A,B,D$ and $F$, we also find
 \begin{eqnarray}\label{eq:papp1eq5}
\frac{dL_4(\cA)}{d\p_2} & = & \frac{4AC}{\sqrt{2 A^2C + 1}}  \frac{dA} {d\p_2}
\nonumber\\
\frac{dL_1(\cA)}{d\p_2} & = & \lp-\frac{8BC}{L_4(\cA)^2}\frac{dB} {d\p_2} + \frac{8B^2C}{L_4(\cA)^3} \frac{dL_4(\cA)}{d\p_2} \rp L_1(\cA)
\nonumber\\
\frac{dL_5(\cA)}{d\p_2} & = & \sqrt{2}\lp  \lp 2  \frac{dA} {d\p_2} C(AF + B)  + 2AC\lp \frac{dA} {d\p_2} F +A \frac{dF} {d\p_2} +  \frac{dB} {d\p_2}  \rp +  \frac{dF} {d\p_2} \rp  \rp
\nonumber\\
\frac{dL_6(\cA)}{d\p_2} & = & \sqrt{2}\lp  \lp 2  \frac{dA} {d\p_2} C(AD + B)  + 2AC\lp \frac{dA} {d\p_2} D +A \frac{dD} {d\p_2} +  \frac{dB} {d\p_2}  \rp +  \frac{dD} {d\p_2} \rp  \rp
\nonumber\\
\frac{dL_2(\cA)}{d\p_2} & = & \frac{2}{\sqrt{\pi}} e^{-\lp \frac{L_5(\cA)}{L_4(\cA)}\rp^2}  \lp\frac{dL_5(\cA)}{d\p_2}\frac{1}{L_4(\cA)}- \frac{L_5(\cA)}{L_4(\cA)^2} \frac{dL_4(\cA)}{d\p_2} \rp
\nonumber\\
\frac{dL_3(\cA)}{d\p_2} & = & \frac{2}{\sqrt{\pi}} e^{-\lp \frac{L_6(\cA)}{L_4(\cA)}\rp^2}  \lp\frac{dL_6(\cA)}{d\p_2}\frac{1}{L_4(\cA)}- \frac{L_6(\cA)}{L_4(\cA)^2} \frac{dL_4(\cA)}{d\p_2} \rp
\nonumber \\
 \frac{dI_1(\cA )}{d\p_2} & = &
\frac{dL_1(\cA)}{d\p_2}\frac{(L_2(\cA)-L_3(\cA))}{L_4(\cA)} + \frac{L_1(\cA)}{L_4(\cA)}\lp\frac{dL_2(\cA)}{d\p_2}-\frac{dL_3(\cA)}{d\p_2}\rp
\nonumber \\
& & - \frac{L_1(\cA)(L_2(\cA)-L_3(\cA))}{L_4(\cA)^2}\frac{dL_4(\cA)}{d\p_2}
\nonumber \\
  \frac{d\hat{f}_{23}\lp I_1(\cA)\rp } {d\p_2}
 & = &
  e^{\frac{\c_2}{4\gamma_{sq}^{(p)}}\nu}  \frac{dI_1(\cA )}{d\p_2}.
\end{eqnarray}
Solving the remaining integrals and combining (\ref{eq:negprac24a0a0}), (\ref{eq:papp1eq3}), and (\ref{eq:papp1eq5}) then gives
 \begin{eqnarray}\label{eq:papp1eq7}
  \frac{d\hat{f}_2^{(2)}} {d\p_2}  =
  \frac{d\hat{f}_{21}\lp \bar{a}_0^{(2)}\rp } {d\p_2}
  +   \frac{d\hat{f}_{22}\lp \bar{a}_0^{(2)}\rp } {d\p_2}
   +    \frac{d\hat{f}_{23}\lp I_1(\cA^{(2)})\rp } {d\p_2}.
\end{eqnarray}
From (\ref{eq:app1eq1}), we then easily obtain
\begin{eqnarray}\label{eq:papp1eq8}
 \frac{dI_L^{(2)}}{d\p_2} & = &
\frac{\alpha}{\c_2}\mE_{{\mathcal U}_3} \lp \frac{1}{\hat{f}_2^{(2)}}  \frac{d\hat{f}_2^{(2)}} {d\p_2}
 \rp.
     \end{eqnarray}
A combination of (\ref{eq:app1eq1}), (\ref{eq:app1eq1a0}), (\ref{eq:papp1eq2}), (\ref{eq:papp1eq7}), and (\ref{eq:papp1eq8})  then gives
\begin{eqnarray}\label{eq:papp1eq9}
  \frac{d  \bar{\psi}_{rd}^{(2)}}{d\p_2}
 & = &
 - \frac{1}{2}
\c_2\q_2
   -
\frac{\alpha}{\c_2}\mE_{{\mathcal U}_3} \lp \frac{1}{\hat{f}_2^{(2)}}  \frac{d\hat{f}_2^{(2)}} {d\p_2}
  \rp.
     \end{eqnarray}

\subsection{$\nu$ derivative}

We trivially observe
\begin{eqnarray}\label{eq:nuapp1eq2}
\frac{dI_R^{(2)} }{d\nu}& = &
 0,
     \end{eqnarray}
and recalling on (\ref{eq:negprac24a0a0})  write
 \begin{eqnarray}\label{eq:nuapp1eq3}
  \frac{d\bar{a}_{0}^{(2)} } {d\nu} &  = & 0
\nonumber \\
 \frac{d\bar{a}_{2}^{(2)} } {d\nu} &  = & 1/\sqrt{2\nu}/\sqrt{1-\p_2}
 \nonumber \\
  \frac{d\hat{f}_{21}\lp \bar{a}_0^{(2)}\rp } {d\nu} &  = & \frac{\c_2}{4\gamma_{sq}^{(p)}} e^{\frac{\c_2}{4\gamma_{sq}^{(p)}}\nu}
\lp \frac{1}{2} - \frac{1}{2}\erfc\lp  \frac{ \bar{a}_0^{(2)}}{\sqrt{2}} \rp \rp
\nonumber \\
 \frac{d\hat{f}_{22}\lp \bar{a}_2^{(2)}\rp } {d\nu} & = &  -\frac{\c_2}{4\gamma_{sq}^{(p)}}  e^{-\frac{\c_2}{4\gamma_{sq}^{(p)}}\nu} \frac{1}{2}\erfc\lp  \frac{\bar{a}_2^{(2)}}{\sqrt{2}} \rp
 - e^{-\frac{\c_2}{4\gamma_{sq}^{(p)}}\nu}
\frac{1}{\sqrt{\pi}} e^{-\lp \frac{ \bar{a}_2^{(2)}}{\sqrt{2}}   \rp^2}  \frac{1}{\sqrt{2}}  \frac{d\bar{a}_2^{(2)} }{d\nu}
 \nonumber \\
  \frac{dD^{(2)}} {d\nu} & = &    \frac{d\bar{a}_{0}^{(2)} } {d\nu}
   \nonumber \\
  \frac{dF^{(2)}} {d\nu} & = &    \frac{d\bar{a}_{2}^{(2)} } {d\nu}.
\end{eqnarray}
Recalling further on (\ref{eq:app1eq4a0}) and (\ref{eq:app1eq4}) and  assuming that $\cA$ depends on $\nu$ only via $D$ and $F$, we also find
 \begin{eqnarray}\label{eq:nupapp1eq5}
\frac{dL_4(\cA)}{d\nu} & = & 0
\nonumber\\
\frac{dL_1(\cA)}{d\nu} & = & 0
\nonumber\\
\frac{dL_5(\cA)}{d\nu} & = & \sqrt{2}\lp   2AC\lp A \frac{dF} {d\nu}  \rp +  \frac{dF} {d\nu} \rp
\nonumber\\
\frac{dL_6(\cA)}{d\nu} & = & \sqrt{2}\lp   2AC\lp A \frac{dD} {d\nu}  \rp +  \frac{dD} {d\nu} \rp
\nonumber\\
\frac{dL_2(\cA)}{d\nu} & = & \frac{2}{\sqrt{\pi}} e^{-\lp \frac{L_5(\cA)}{L_4(\cA)}\rp^2}  \lp\frac{dL_5(\cA)}{d\nu}\frac{1}{L_4(\cA)}- \frac{L_5(\cA)}{L_4(\cA)^2} \frac{dL_4(\cA)}{d\nu} \rp
\nonumber\\
\frac{dL_3(\cA)}{d\nu} & = & \frac{2}{\sqrt{\pi}} e^{-\lp \frac{L_6(\cA)}{L_4(\cA)}\rp^2}  \lp\frac{dL_6(\cA)}{d\nu }\frac{1}{L_4(\cA)}- \frac{L_6(\cA)}{L_4(\cA)^2} \frac{dL_4(\cA)}{d\nu} \rp
\nonumber \\
 \frac{dI_1(\cA )}{d\nu} & = &
\frac{dL_1(\cA)}{d\nu}\frac{(L_2(\cA)-L_3(\cA))}{L_4(\cA)} + \frac{L_1(\cA)}{L_4(\cA)}\lp\frac{dL_2(\cA)}{d\nu}-\frac{dL_3(\cA)}{d\nu}\rp
\nonumber \\
& & - \frac{L_1(\cA)(L_2(\cA)-L_3(\cA))}{L_4(\cA)^2}\frac{dL_4(\cA)}{d\nu}
\nonumber \\
  \frac{d\hat{f}_{23}\lp I_1(\cA)\rp } {d\nu}
 & = &
  \frac{\c_2}{4\gamma_{sq}^{(p)}}
  e^{\frac{\c_2}{4\gamma_{sq}^{(p)}}\nu}  I_1(\cA)
+  e^{\frac{\c_2}{4\gamma_{sq}^{(p)}}\nu}  \frac{dI_1(\cA )}{d\nu}.
\end{eqnarray}
After solving the remaining integrals, a combination of (\ref{eq:negprac24a0a0}), (\ref{eq:nuapp1eq3}), and (\ref{eq:nupapp1eq5}) then gives
 \begin{eqnarray}\label{eq:nuapp1eq7}
  \frac{d\hat{f}_2^{(2)}} {d\nu}  =
  \frac{d\hat{f}_{21}\lp \bar{a}_0^{(2)}\rp } {d\nu}
  +   \frac{d\hat{f}_{22}\lp \bar{a}_0^{(2)}\rp } {d\nu}
   +    \frac{d\hat{f}_{23}\lp I_1(\cA^{(2)})\rp } {d\nu}.
\end{eqnarray}
One then from (\ref{eq:app1eq1})  easily finds
\begin{eqnarray}\label{eq:nuapp1eq8}
 \frac{dI_L^{(2)}}{d\nu} & = &
\frac{\alpha}{\c_2}\mE_{{\mathcal U}_3} \lp \frac{1}{\hat{f}_2^{(2)}}  \frac{d\hat{f}_2^{(2)}} {d\nu}
 \rp.
     \end{eqnarray}
Combining (\ref{eq:app1eq1}), (\ref{eq:app1eq1a0}), (\ref{eq:nuapp1eq2}), (\ref{eq:nuapp1eq7}), and (\ref{eq:nuapp1eq8}) we finally obtain
\begin{eqnarray}\label{eq:papp1eq9}
  \frac{d  \bar{\psi}_{rd}^{(2)}}{d\nu}
 & = &
    -\frac{2-\alpha}{4\gamma_{sq}^{(p)}}
-
\frac{\alpha}{\c_2}\mE_{{\mathcal U}_3} \lp \frac{1}{\hat{f}_2^{(2)}}  \frac{d\hat{f}_2^{(2)}} {d\nu}
  \rp.
     \end{eqnarray}

\subsection{$\q_2$ and $\gamma_{sq}^{(q)}$ derivative}

Both $\q_2$ and $\gamma_{sq}^{(q)}$ derivatives have been computed in \cite{Stojnicnegsphflrdt23}. For completeness we restate the obtained results
\begin{eqnarray}\label{eq:2levder1}
   \frac{d\bar{\psi}_{rd}^{(2)}  }{d\q_2}
   & = &  \c_2\lp -\frac{1}{2}
\p_2
+\frac{\q_2}{2(2\gamma_{sq}^{(p)}-\c_2(1-\q_2))^2}\rp
\nonumber \\
   \frac{d\bar{\psi}_{rd}^{(2)} }{d\gamma_{sq}^{(q)}}
    & = &    -1-\lp -\frac{1-\q_2}{2\gamma_{sq}^{(q)}(2\gamma_{sq}^{(q)}-\c_2(1-\q_2))}-\frac{\q_2}{(2\gamma_{sq}^{(q)}-\c_2(1-\q_2))^2}\rp.
     \end{eqnarray}

\subsection{$\gamma_{sq}^{(p)}$ derivative}

It is not that difficult to see that
\begin{eqnarray} \label{eq:gamaapp1eq1}
   \frac{d\hat{f}_2^{(2)}}{d\gamma_{sq}^{(p)}}
   =-\frac{\c_2}{\gamma_{sq}^{(p)}}
   \frac{d\hat{f}_2^{(2)}}{d\c_2}.
 \end{eqnarray}
From (\ref{eq:app1eq1}) and (\ref{eq:app1eq1a0}) one then has
\begin{eqnarray} \label{eq:gamaapp1eq2}
  \frac{d  \bar{\psi}_{rd}^{(2)}}{d\gamma_{sq}^{(p)}  }
 & = &
   1 + \frac{\nu(2-\alpha)}{4\lp\gamma_{sq}^{(p)}\rp^2}
-
\frac{\alpha}{\c_2}\mE_{{\mathcal U}_3} \lp \frac{1}{\hat{f}_2^{(2)}}    \frac{d\hat{f}_2^{(2)}}{d\gamma_{sq}^{(p)}}
  \rp.
 \end{eqnarray}

\section{Proof of Corollary \ref{cor:closedformrel1}}
\label{sec:appB}

\begin{proof}[Proof of Corollary \ref{cor:closedformrel1}]
The first two equalities follow from (\ref{eq:2levder1}) (detailed derivations are presented in equations (84)-(90) in \cite{Stojnicnegsphflrdt23}). We now focus on the third equality. To that end, we start by noting that a combination of (\ref{eq:gamaapp1eq1})
and (\ref{eq:gamaapp1eq2}) gives
\begin{eqnarray} \label{eq:aapp2eq1}
  \frac{d  \bar{\psi}_{rd}^{(2)}}{d\gamma_{sq}^{(p)}  }
 & = &
   1 + \frac{\nu(2-\alpha)}{4\lp\gamma_{sq}^{(p)}\rp^2}
+
\frac{\c_2}{\gamma_{sq}^{(p)}}
\frac{\alpha}{\c_2}\mE_{{\mathcal U}_3} \lp \frac{1}{\hat{f}_2^{(2)}}    \frac{d\hat{f}_2^{(2)}}{d\c_2}
  \rp.
 \end{eqnarray}
After equalling the derivative to zero one further finds
\begin{eqnarray} \label{eq:aapp2eq2}
-\frac{\gamma_{sq}^{(p)}} {\c_2}\lp  1 + \frac{\nu(2-\alpha)}{4\lp\gamma_{sq}^{(p)}\rp^2} \rp
 & = &
\frac{\alpha}{\c_2}\mE_{{\mathcal U}_3} \lp \frac{1}{\hat{f}_2^{(2)}}    \frac{d\hat{f}_2^{(2)}}{d\c_2}
  \rp.
 \end{eqnarray}
From (\ref{eq:app1eq1a0}) one (for $\bar{\psi}_{rd}^{(2)}=0$) also has
\begin{eqnarray}\label{eq:aapp2eq3}
 \frac{\alpha}{\c_2}\mE_{{\mathcal U}_3}\log\lp \hat{f}_2^{(2)}
 \rp & = &
  \frac{1}{2}
(1-\p_2\q_2)\c_2     - I_R^{(2)}
  +\gamma_{sq}^{(p)} -\frac{\nu(2-\alpha)}{4\gamma_{sq}^{(p)}}.
     \end{eqnarray}
Utilizing (\ref{eq:app1eq9}), we further write
\begin{eqnarray}\label{eq:app2eq4}
  \frac{d  \bar{\psi}_{rd}^{(2)}}{d\c_2}
 & = &
  \frac{1}{2}
(1-\p_2\q_2)
\nonumber \\
& &
-
\lp \frac{1}{2\c_2^2} \log \lp \frac{2\gamma_{sq}^{(q)}-\c_2(1-\q_2)}{2\gamma_{sq}^{(q)}} \rp
 + \frac{1-\q_2}{2\c_2(2\gamma_{sq}^{(q)}-\c_2(1-\q_2))}
 + \frac{\q_2(1-\q_2)}{2(2\gamma_{sq}^{(q)}-\c_2(1-\q_2))^2} \rp
\nonumber \\
 & &
  - \lp   -\frac{\alpha}{\c_2^2}\mE_{{\mathcal U}_3} \log\lp  \hat{f}_2^{(2)}
 \rp
 +
\frac{\alpha}{\c_2}\mE_{{\mathcal U}_3} \lp \frac{1}{\hat{f}_2^{(2)}}  \frac{d\hat{f}_2^{(2)}} {d\c_2}
 \rp  \rp
 \nonumber \\
 & = &
  \frac{1}{2}
(1-\p_2\q_2)
\nonumber \\
& &
-
\lp \frac{1}{2\c_2^2} \log \lp \frac{2\gamma_{sq}^{(q)}-\c_2(1-\q_2)}{2\gamma_{sq}^{(q)}} \rp
 + \frac{1-\q_2}{2\c_2(2\gamma_{sq}^{(q)}-\c_2(1-\q_2))}
 + \frac{\q_2(1-\q_2)}{2(2\gamma_{sq}^{(q)}-\c_2(1-\q_2))^2} \rp
\nonumber \\
 & &
  - \lp   -\frac{1}{\c_2}
  \lp        \frac{1}{2}
(1-\p_2\q_2)\c_2     - I_R^{(2)}
  +\gamma_{sq}^{(p)} -\frac{\nu(2-\alpha)}{4\gamma_{sq}^{(p)}}        \rp
 -\frac{\gamma_{sq}^{(p)}} {\c_2}\lp  1 + \frac{\nu(2-\alpha)}{4\lp\gamma_{sq}^{(p)}\rp^2} \rp
 \rp
 \nonumber \\
 & = &
 (1-\p_2\q_2)
\nonumber \\
& &
-
\lp \frac{1}{2\c_2^2} \log \lp \frac{2\gamma_{sq}^{(q)}-\c_2(1-\q_2)}{2\gamma_{sq}^{(q)}} \rp
 + \frac{1-\q_2}{2\c_2(2\gamma_{sq}^{(q)}-\c_2(1-\q_2))}
 + \frac{\q_2(1-\q_2)}{2(2\gamma_{sq}^{(q)}-\c_2(1-\q_2))^2} \rp
\nonumber \\
 & &
  - \lp   -\frac{1}{\c_2}
  \lp       - I_R^{(2)}
  + 2\gamma_{sq}^{(p)}        \rp
 \rp
 \nonumber \\
 & = &
 (1-\p_2\q_2)
\nonumber \\
& &
-
\lp \frac{1}{2\c_2^2} \log \lp \frac{2\gamma_{sq}^{(q)}-\c_2(1-\q_2)}{2\gamma_{sq}^{(q)}} \rp
 + \frac{1-\q_2}{2\c_2(2\gamma_{sq}^{(q)}-\c_2(1-\q_2))}
 + \frac{\q_2(1-\q_2)}{2(2\gamma_{sq}^{(q)}-\c_2(1-\q_2))^2} \rp
\nonumber \\
 & &
  - \lp   -\frac{1}{\c_2}
  \lp       -
  \lp
  \gamma_{sq}^{(q)}  + \Bigg(\Bigg. -\frac{1}{2\c_2} \log \lp \frac{2\gamma_{sq}^{(q)}-\c_2(1-\q_2)}{2\gamma_{sq}^{(q)}} \rp  +  \frac{\q_2}{2(2\gamma_{sq}^{(q)}-\c_2(1-\q_2))}   \Bigg.\Bigg)
  \rp
   + 2\gamma_{sq}^{(p)}        \rp
 \rp
 \nonumber \\
 & = &
 (1-\p_2\q_2)
\nonumber \\
& &
-
\lp
   \frac{1}{2\c_2(2\gamma_{sq}^{(q)}-\c_2(1-\q_2))}
 + \frac{\q_2(1-\q_2)}{2(2\gamma_{sq}^{(q)}-\c_2(1-\q_2))^2} \rp
  - \lp   -\frac{1}{\c_2}
  \lp       -
  \lp
  \gamma_{sq}^{(q)}
  \rp
   + 2\gamma_{sq}^{(p)}        \rp
 \rp.
     \end{eqnarray}
Setting to zero the derivatives in (\ref{eq:2levder1}) gives
\begin{eqnarray}\label{eq:app2eq5}
  \frac{\p_2}{\q_2}  & = &
\frac{1}{(2\gamma_{sq}^{(p)}-\c_2(1-\q_2))^2}
\nonumber \\
\frac{1}{2 (2\gamma_{sq}^{(q)}-\c_2(1-\q_2)  }
& = & \gamma_{sq}^{(q)}
\frac{1-\p_2}{1-\q_2}.
     \end{eqnarray}
One then easily finds
\begin{eqnarray}\label{eq:app2eq6}
\c_2
& = & 2 \gamma_{sq}^{(q)}\frac{1}{1-\q_2}  - \frac{1}{2\gamma_{sq}^{(q)}}
\frac{1}{1-\p_2}.
     \end{eqnarray}
Setting to zero the derivative in (\ref{eq:app2eq4}) gives
\begin{eqnarray}\label{eq:app2eq7}
 - 2\gamma_{sq}^{(p)} & = &
\c_2 (1-\p_2\q_2)
 -
\lp
   \frac{1}{2(2\gamma_{sq}^{(q)}-\c_2(1-\q_2))}
 + \frac{\c_2\q_2(1-\q_2)}{2(2\gamma_{sq}^{(q)}-\c_2(1-\q_2))^2} \rp
  -    \gamma_{sq}^{(q)}.
     \end{eqnarray}
Combining (\ref{eq:app2eq5}) and (\ref{eq:app2eq7}) we obtain
\begin{eqnarray}\label{eq:app2eq8}
 - 2\gamma_{sq}^{(p)} & = &
\c_2 (1-\p_2\q_2)
 -
\lp
\gamma_{sq}^{(q)}
\frac{1-\p_2}{1-\q_2}
 + \frac{\c_2\p_2(1-\q_2)}{2 } \rp
  -    \gamma_{sq}^{(q)}.
     \end{eqnarray}
Plugging $\c_2$ from (\ref{eq:app2eq6}) into (\ref{eq:app2eq8}) gives
\begin{eqnarray}\label{eq:app2eq9}
 - 2\gamma_{sq}^{(p)} & = &
\lp   2 \gamma_{sq}^{(q)}\frac{1}{1-\q_2}  - \frac{1}{2\gamma_{sq}^{(q)}}
\frac{1}{1-\p_2}  \rp (1-\p_2\q_2)
\nonumber \\
& &
 -
\lp
\gamma_{sq}^{(q)}
\frac{1-\p_2}{1-\q_2}
 + \lp 2 \gamma_{sq}^{(q)}\frac{1}{1-\q_2}  - \frac{1}{2\gamma_{sq}^{(q)}}
\frac{1}{1-\p_2}  \rp\frac{\p_2(1-\q_2)}{2 } \rp
  -    \gamma_{sq}^{(q)}
  \nonumber  \\
  & = &
   \gamma_{sq}^{(q)}\frac{2(1-\p_2\q_2)  -(1-\p_2) -\p_2(1-\q_2 ) - (1-\q_2))  }{1-\q_2}
-\frac{1}{4\gamma_{sq}^{(q)}} \frac{2(1-\p_2\q_2)  -\p_2(1-\q_2) }{1-\p_2}
  \nonumber  \\
  & = &
   \gamma_{sq}^{(q)}\frac{(\q_2-\p_2\q_2 )  }{1-\q_2}
-\frac{1}{4\gamma_{sq}^{(q)}} \frac{2-\p_2\q_2  -\p_2 }{1-\p_2}
  \nonumber  \\
  & = &
   \frac{1}{2}\sqrt{\frac{\p_2}{\q_2}}\q_2
-\frac{1}{4\gamma_{sq}^{(q)}} \frac{\p_2-\p_2\q_2  }{1-\p_2}
-\frac{1}{2\gamma_{sq}^{(q)}}
\nonumber \\
  & = &
   \frac{1}{2}\sqrt{\frac{\p_2}{\q_2}}\q_2
-   \frac{1}{2}\sqrt{\frac{\q_2}{\p_2}}\p_2
-\frac{1}{2\gamma_{sq}^{(q)}}
 \nonumber \\
  & = &
 -\frac{1}{2\gamma_{sq}^{(q)}}.
     \end{eqnarray}
 Connecting beginning and end in the last sequence of equalities and keeping in mind that $\hat{\gamma}_{sq}^{(q)}$ and $\hat{\gamma}_{sq}^{(p)}$ are precisely the solutions of zero-derivatives equations completes the proof.
\end{proof}

\section{Derivatives -- Third level of lifting}
\label{sec:appC}

We recall on (\ref{eq:3negprac25})
\begin{eqnarray}\label{eq:Capp1negprac25}
    \bar{\psi}_{rd}^{(3)}
 & = &   \frac{1}{2}
(1-\p_2\q_2)\c_2+ \frac{1}{2}
(\p_2\q_2-\p_3\q_3)\c_3
 -\gamma_{sq}^{(q)}  - \Bigg(\Bigg. -\frac{1}{2\c_2} \log \lp \frac{2\gamma_{sq}^{(q)}-\c_2 (1-\q_2)}{2\gamma_{sq}^{(q)}} \rp
   \nonumber \\
& &
  -\frac{1}{2\c_3} \log \lp \frac{2\gamma_{sq}^{(q)}-\c_2 (1-\q_2)-\c_3 (\q_2-\q_3)}{2\gamma_{sq}^{(q)}-\q_2 (1-\q_2)} \rp  +  \frac{\q_3}{2(2\gamma_{sq}^{(q)}-\c_2 (1-\q_2)-\c_3(\q_2-\q_3))}   \Bigg.\Bigg)
     \nonumber \\
& & + \gamma_{sq}^{(p)} - \frac{\nu(2-\alpha)}{\gamma_{sq}^{(p)}} - \frac{\alpha}{\c_3}\mE_{{\mathcal U}_4}\log  \lp \mE_{{\mathcal U}_3} \lp \hat{f}_2^{(3)}\rp^{\frac{\c_3}{\c_2}}  \rp,
     \end{eqnarray}
After  setting
\begin{eqnarray}\label{eq:Capp1eq1}
I_R^{(2)} & = &
\gamma_{sq}^{(q)}  + \Bigg(\Bigg. -\frac{1}{2\c_2} \log \lp \frac{2\gamma_{sq}^{(q)}-\c_2 (1-\q_2)}{2\gamma_{sq}^{(q)}} \rp
   \nonumber \\
& &
  -\frac{1}{2\c_3} \log \lp \frac{2\gamma_{sq}^{(q)}-\c_2 (1-\q_2)-\c_3 (\q_2-\q_3)}{2\gamma_{sq}^{(q)}-\q_2 (1-\q_2)} \rp  +  \frac{\q_3}{2(2\gamma_{sq}^{(q)}-\c_2 (1-\q_2)-\c_3(\q_2-\q_3))}   \Bigg.\Bigg)
\nonumber \\
I_L^{(2)} & = &
\frac{\alpha}{\c_3}\mE_{{\mathcal U}_4}\log  \lp \mE_{{\mathcal U}_3} \lp \hat{f}_2^{(3)}\rp^{\frac{\c_3}{\c_2}}  \rp,
     \end{eqnarray}
one can trivially rewrite (\ref{eq:Capp1negprac25})
\begin{eqnarray}\label{eq:Capp1eq1a0}
    \bar{\psi}_{rd}^{(3)}
 & = &
\frac{1}{2}
(1-\p_2\q_2)\c_2+ \frac{1}{2}
(\p_2\q_2-\p_3\q_3)\c_3  - I_R^{(3)}
  +\gamma_{sq}^{(p)} -\frac{\nu(2-\alpha)}{4\gamma_{sq}^{(p)}}
 -  I_L^{(3)}.
     \end{eqnarray}

\subsection{$\c$ derivatives}

We separately compute $\c_2$ and $\c_3$ derivatives.

\vspace{.1in}
\noindent \underline{\textbf{\emph{\red{ 1) $\c_2$ derivative:}}}}  We first find
\begin{eqnarray}\label{eq:Capp1eq2}
\frac{dI_R^{(3)} }{d\c_2}& = &
\frac{1}{2\c_2^2}\log\lp \frac{ 2\gamma_{sq}^{(q)} -\c_2(1-\q_2)}   {2\gamma_{sq}^{(q)} } \rp
+ \frac{1-\q_2}{2\c_2(2\gamma_{sq}^{(q)}-\c_2(1-\q_2))} \nonumber \\
& &
+ \frac{1-\q_2}{2\c_3(2\gamma_{sq}^{(q)}-\c_2(1-\q_2) -\c_3(\q_2-\q_3))}
- \frac{1-\q_2}{2\c_3(2\gamma_{sq}^{(q)}-\c_2(1-\q_2))}
\nonumber \\
& &
+
\frac{(1-\q_2)\q_3}{2(2\gamma_{sq}^{(q)}-\c_2(1-\q_2)-\c_3(\q_2-\q_3))^2},
     \end{eqnarray}
and recalling on (\ref{eq:3negprac24a0a0}) further write
 \begin{eqnarray}\label{eq:Capp1eq3}
  \frac{d\hat{f}_{21}\lp \bar{a}_0^{(3)}\rp } {d\c_2} &  = & \frac{\nu}{4\gamma_{sq}^{(p)}} e^{\frac{\c_2}{4\gamma_{sq}^{(p)}}\nu}
\lp \frac{1}{2} - \frac{1}{2}\erfc\lp  \frac{ \bar{a}_0^{(3)}}{\sqrt{2}} \rp \rp
\nonumber \\
 \frac{d\hat{f}_{22}\lp \bar{a}_2^{(3)}\rp } {d\c_2} & = &  -\frac{\nu}{4\gamma_{sq}^{(p)}}  e^{-\frac{\c_2}{4\gamma_{sq}^{(p)}}\nu} \frac{1}{2}\erfc\lp  \frac{\bar{a}_2^{(3)}}{\sqrt{2}} \rp
 \nonumber \\
 \frac{dC^{(3)}} {d\c_2} & = &  \frac{1}{4\gamma_{sq}^{(p)}}.
 \end{eqnarray}
 Taking
 \begin{eqnarray}\label{eq:Capp1eq4a0a0}
\cA^{(i)} = [A^{(i)},B^{(i)},C^{(i)},D^{(i)},F^{(i)}], i=3,
\end{eqnarray}
and solving the remaining integrals, we obtain the following third level analogue to second level (\ref{eq:app1eq7})
 \begin{eqnarray}\label{eq:Capp1eq7}
  \frac{d\hat{f}_2^{(3)}} {d\c_2}  =
  \frac{d\hat{f}_{21}\lp \bar{a}_0^{(3)}\rp } {d\c_2}
  +   \frac{d\hat{f}_{22}\lp \bar{a}_0^{(3)}\rp } {d\c_2}
   +    \frac{d\hat{f}_{23}\lp I_1(\cA^{(3)})\rp } {d\c_2}.
\end{eqnarray}
From (\ref{eq:Capp1eq1}), one then finds
\begin{eqnarray}\label{eq:Capp1eq8}
 \frac{dI_L^{(3)}}{d\c_2} & = &
 \frac{\alpha}{\c_3}\mE_{{\mathcal U}_4} \lp \frac{1}{\mE_{{\mathcal U}_3} \lp \hat{f}_3^{(2)} \rp^{\frac{\c_3}{\c_2}}  } \mE_{{\mathcal U}_3}
 \lp  \lp -\frac{\c_3}{\c_2^2}\log(\hat{f}_2^{(3)}) + \frac{\c_3}{\c_2}\frac{1}{\hat{f}_2^{(3)}}  \frac{d\hat{f}_2^{(3)}} {d\c_2}  \rp
 \lp \hat{f}_2^{(3)}\rp^{\c_3/\c_2}  \rp
 \rp.
     \end{eqnarray}
A combination of (\ref{eq:Capp1eq1}), (\ref{eq:Capp1eq1a0}),,(\ref{eq:Capp1eq2}), (\ref{eq:Capp1eq7}), and (\ref{eq:Capp1eq8})  then gives
\begin{eqnarray}\label{eq:Capp1eq9}
  \frac{d  \bar{\psi}_{rd}^{(3)}}{d\c_2}
 & = &
  \frac{1}{2}
(1-\p_2\q_2)
\nonumber \\
& &
-
\Bigg (\Bigg.
\frac{1}{2\c_2^2}\log\lp \frac{ 2\gamma_{sq}^{(q)} -\c_2(1-\q_2)}   {2\gamma_{sq}^{(q)} } \rp
+ \frac{1-\q_2}{2\c_2(2\gamma_{sq}^{(q)}-\c_2(1-\q_2))} \nonumber \\
& &
+ \frac{1-\q_2}{2\c_3(2\gamma_{sq}^{(q)}-\c_2(1-\q_2) -\c_3(\q_2-\q_3))}
- \frac{1-\q_2}{2\c_3(2\gamma_{sq}^{(q)}-\c_2(1-\q_2))}
\nonumber \\
& &
+
\frac{(1-\q_2)\q_3}{2(2\gamma_{sq}^{(q)}-\c_2(1-\q_2)-\c_3(\q_2-\q_3))^2}
 \Bigg. \Bigg )
\nonumber \\
 & &
  - \lp   \frac{\alpha}{\c_3}\mE_{{\mathcal U}_4} \lp \frac{1}{\mE_{{\mathcal U}_3} \lp \hat{f}_3^{(2)} \rp^{\frac{\c_3}{\c_2}}  } \mE_{{\mathcal U}_3}
 \lp  \lp -\frac{\c_3}{\c_2^2}\log(\hat{f}_2^{(3)}) + \frac{\c_3}{\c_2}\frac{1}{\hat{f}_2^{(3)}}  \frac{d\hat{f}_2^{(3)}} {d\c_2}  \rp
 \lp \hat{f}_2^{(3)}\rp^{\c_3/\c_2}  \rp
 \rp
  \rp.
     \end{eqnarray}

\vspace{.1in}
\noindent \underline{\textbf{\emph{\red{ 2) $\c_3$ derivative:}}}}  We first find
\begin{eqnarray}\label{eq:CXapp1eq2}
\frac{dI_R^{(3)} }{d\c_3}& = &
1/2/\c_3^2\log \lp  \frac{(2\gamma_{sq}^{(q)}-\c_2(1-\q_2) -\c_3(\q_2-\q_3))}{(2\gamma_{sq}^{(q)}-\c_2(1-\q_2))}  \rp
 + \frac{(\q_2-\q_3)}{2\c_3(2\gamma_{sq}^{(q)}-\c_2(1-\q_2) -\c_3(\q_2-\q_3))}
\nonumber \\
& &
+\frac{(\q_2-\q_3)\q_3}{2(2\gamma_{sq}^{(q)}-\c_2(1-\q_2)-\c_3(\q_2-\q_3))^2},
     \end{eqnarray}
and after noting that
 \begin{eqnarray}\label{eq:CXapp1eq7}
  \frac{d\hat{f}_2^{(3)}} {d\c_3}  = 0,
  \end{eqnarray}
obtain from (\ref{eq:Capp1eq1}) the following $\c_3$ analogue to (\ref{eq:Capp1eq8})
\begin{eqnarray}\label{eq:CXapp1eq8}
 \frac{dI_L^{(3)}}{d\c_3} & = &
 -\frac{\alpha}{\c_3^2}\mE_{{\mathcal U}_4}\log  \lp \mE_{{\mathcal U}_3} \lp \hat{f}_2^{(3)}\rp^{\frac{\c_3}{\c_2}}  \rp
 +
 \frac{\alpha}{\c_3}\mE_{{\mathcal U}_4} \lp \frac{1}{\mE_{{\mathcal U}_3} \lp \hat{f}_3^{(2)} \rp^{\frac{\c_3}{\c_2}}  } \mE_{{\mathcal U}_3}
 \lp  \lp \frac{1}{\c_2}\log(\hat{f}_2^{(3)})  \rp
 \lp \hat{f}_2^{(3)}\rp^{\c_3/\c_2}  \rp
 \rp.\nonumber \\
     \end{eqnarray}
A combination of (\ref{eq:Capp1eq1}), (\ref{eq:Capp1eq1a0}),,(\ref{eq:CXapp1eq2}), (\ref{eq:CXapp1eq7}), and (\ref{eq:CXapp1eq8})  then gives
\begin{align}\label{eq:CXapp1eq9}
  \frac{d  \bar{\psi}_{rd}^{(3)}}{d\c_3}
 & =
  \frac{1}{2}
(\p_2\q_2-\p_3\q_3)
\nonumber \\
& \quad
-
\Bigg (\Bigg.
\frac{1}{2\c_3^2}\log \lp  \frac{(2\gamma_{sq}^{(q)}-\c_2(1-\q_2) -\c_3(\q_2-\q_3))}{(2\gamma_{sq}^{(q)}-\c_2(1-\q_2))}  \rp
 + \frac{(\q_2-\q_3)}{2\c_3(2\gamma_{sq}^{(q)}-\c_2(1-\q_2) -\c_3(\q_2-\q_3))}
\nonumber \\
& \quad
+\frac{(\q_2-\q_3)\q_3}{2(2\gamma_{sq}^{(q)}-\c_2(1-\q_2)-\c_3(\q_2-\q_3))^2}
 \Bigg. \Bigg )
\nonumber \\
 & \quad
  - \lp
  -\frac{\alpha}{\c_3^2}\mE_{{\mathcal U}_4}\log  \lp \mE_{{\mathcal U}_3} \lp \hat{f}_2^{(3)}\rp^{\frac{\c_3}{\c_2}}  \rp
 +
 \frac{\alpha}{\c_3}\mE_{{\mathcal U}_4} \lp \frac{1}{\mE_{{\mathcal U}_3} \lp \hat{f}_2^{(3)} \rp^{\frac{\c_3}{\c_2}}  } \mE_{{\mathcal U}_3}
 \lp  \lp \frac{1}{\c_2}\log(\hat{f}_2^{(3)})  \rp
 \lp \hat{f}_2^{(3)}\rp^{\c_3/\c_2}  \rp
 \rp
   \rp.
     \end{align}

\subsection{$\p$ derivatives}

Following the above trend, we below separately compute both $\p_2$ and $\p_3$ derivatives.

\vspace{.1in}
\noindent \underline{\textbf{\emph{\red{ 1) $\p_2$ derivative:}}}} We trivially observe
\begin{eqnarray}\label{eq:Cpapp1eq2}
\frac{dI_R^{(3)} }{d\p_2}& = &
 0,
     \end{eqnarray}
and recalling on (\ref{eq:negprac24a0a0})  write
 \begin{eqnarray}\label{eq:Cpapp1eq3}
  \frac{d\bar{a}_{0}^{(3)} } {d\p_2} &  = &
  -(\u_1^{(2,3)}/2/\sqrt{\p_2-\p_3} )/\sqrt{1-\p_2}  -1/2(\u_1^{(2,3)}\sqrt{\p_2-\p_3} + \u_1^{(2,4)}\sqrt{\p_3})/\sqrt{1-\p_2}^3
  \nonumber \\
\frac{d\bar{a}_{2}^{(3)} } {d\p_2} &  = &
-(\u_1^{(2,3)}/2/\sqrt{\p_2-\p_3} )/\sqrt{1-\p_2}  +1/2(\sqrt{2\nu} -(\u_1^{(2,3)}\sqrt{\p_2-\p_3} + \u_1^{(2,4)}\sqrt{\p_3}))/\sqrt{1-\p_2}^3
\nonumber \\
  \frac{d\hat{f}_{21}\lp \bar{a}_0^{(3)}\rp } {d\p_2} &  = &  e^{\frac{\c_2}{4\gamma_{sq}^{(p)}}\nu}
\frac{1}{\sqrt{\pi}} e^{-\lp \frac{ \bar{a}_0^{(2)}}{\sqrt{2}}   \rp^2}  \frac{1}{\sqrt{2}}  \frac{d\bar{a}_0^{(3)} }{d\p_2}
\nonumber \\
 \frac{d\hat{f}_{22}\lp \bar{a}_2^{(3)}\rp } {d\p_2} & = & - e^{-\frac{\c_2}{4\gamma_{sq}^{(p)}}\nu}
\frac{1}{\sqrt{\pi}} e^{-\lp \frac{ \bar{a}_2^{(2)}}{\sqrt{2}}   \rp^2}  \frac{1}{\sqrt{2}}  \frac{d\bar{a}_2^{(3)} }{d\p_2}   \nonumber \\
 \frac{dC^{(3)}} {d\p_2} & = &  0
 \nonumber \\
 \frac{dA^{(3)}} {d\p_2} & = &  -1/2/\sqrt{1-\p_2}
 \nonumber \\
 \frac{dB^{(3)}} {d\p_2} & = &  \u_1^{(2,3)}/2/\sqrt{\p_2-\p_3}
 \nonumber \\
 \frac{dD^{(3)}} {d\p_2} & = &    \frac{d\bar{a}_{0}^{(3)} } {d\p_2}
 \nonumber \\
 \frac{dF^{(3)}} {d\p_2} & = &    \frac{d\bar{a}_{2}^{(3)} } {d\p_2}.
  \end{eqnarray}
 Solving the remaining integrals and combining (\ref{eq:3negprac24a0a0}) and (\ref{eq:Cpapp1eq3}), we have the following third level analogue to (\ref{eq:papp1eq7})
 \begin{eqnarray}\label{eq:Cpapp1eq7}
  \frac{d\hat{f}_2^{(3)}} {d\p_2}  =
  \frac{d\hat{f}_{21}\lp \bar{a}_0^{(3)}\rp } {d\p_2}
  +   \frac{d\hat{f}_{22}\lp \bar{a}_0^{(3)}\rp } {d\p_2}
   +    \frac{d\hat{f}_{23}\lp I_1(\cA^{(3)})\rp } {d\p_2}.
\end{eqnarray}
From (\ref{eq:Capp1eq1}), we then easily obtain
\begin{eqnarray}\label{eq:Cpapp1eq8}
 \frac{dI_L^{(3)}}{d\p_2} & = &
 \frac{\alpha}{\c_3}\mE_{{\mathcal U}_4} \lp \frac{1}{\mE_{{\mathcal U}_3} \lp \hat{f}_3^{(2)} \rp^{\frac{\c_3}{\c_2}}  } \mE_{{\mathcal U}_3}
 \lp  \lp   \frac{\c_3}{\c_2}\frac{1}{\hat{f}_2^{(3)}}  \frac{d\hat{f}_2^{(3)}} {d\p_2}  \rp
 \lp \hat{f}_2^{(3)}\rp^{\c_3/\c_2}  \rp
 \rp.
     \end{eqnarray}
A combination of (\ref{eq:Capp1eq1}), (\ref{eq:Capp1eq1a0}), (\ref{eq:Cpapp1eq2}), (\ref{eq:Cpapp1eq7}), and (\ref{eq:Cpapp1eq8})  then gives
\begin{eqnarray}\label{eq:Cpapp1eq9}
  \frac{d  \bar{\psi}_{rd}^{(3)}}{d\p_2}
 & = &
 - \frac{1}{2}
\c_2\q_2
 + \frac{1}{2}
\c_3\q_3
   -
 \frac{\alpha}{\c_3}\mE_{{\mathcal U}_4} \lp \frac{1}{\mE_{{\mathcal U}_3} \lp \hat{f}_3^{(2)} \rp^{\frac{\c_3}{\c_2}}  } \mE_{{\mathcal U}_3}
 \lp  \lp   \frac{\c_3}{\c_2}\frac{1}{\hat{f}_2^{(3)}}  \frac{d\hat{f}_2^{(3)}} {d\p_2}  \rp
 \lp \hat{f}_2^{(3)}\rp^{\c_3/\c_2}  \rp
 \rp.
     \end{eqnarray}

\vspace{.1in}
\noindent \underline{\textbf{\emph{\red{ 2) $\p_3$ derivative:}}}} We again trivially observe
\begin{eqnarray}\label{eq:CXpapp1eq2}
\frac{dI_R^{(3)} }{d\p_3}& = &
 0,
     \end{eqnarray}
and recalling on (\ref{eq:negprac24a0a0})  write
 \begin{eqnarray}\label{eq:CXpapp1eq3}
  \frac{d\bar{a}_{0}^{(3)} } {d\p_3} &  = &
  -(-\u_1^{(2,3)}/2/\sqrt{\p_2-\p_3} + \u_1^{(2,4)}/2/\sqrt{\p_3})/\sqrt{1-\p_2}
    \nonumber \\
\frac{d\bar{a}_{2}^{(3)} } {d\p_3} &  = &
  -(-\u_1^{(2,3)}/2/\sqrt{\p_2-\p_3} + \u_1^{(2,4)}/2/\sqrt{\p_3})/\sqrt{1-\p_2}
\nonumber \\
  \frac{d\hat{f}_{21}\lp \bar{a}_0^{(3)}\rp } {d\p_3} &  = &  e^{\frac{\c_2}{4\gamma_{sq}^{(p)}}\nu}
\frac{1}{\sqrt{\pi}} e^{-\lp \frac{ \bar{a}_0^{(2)}}{\sqrt{2}}   \rp^2}  \frac{1}{\sqrt{2}}  \frac{d\bar{a}_0^{(3)} }{d\p_3}
\nonumber \\
 \frac{d\hat{f}_{22}\lp \bar{a}_2^{(3)}\rp } {d\p_3} & = & - e^{-\frac{\c_2}{4\gamma_{sq}^{(p)}}\nu}
\frac{1}{\sqrt{\pi}} e^{-\lp \frac{ \bar{a}_2^{(2)}}{\sqrt{2}}   \rp^2}  \frac{1}{\sqrt{2}}  \frac{d\bar{a}_2^{(3)} }{d\p_3}   \nonumber \\
 \frac{dC^{(3)}} {d\p_3} & = &  0
 \nonumber \\
 \frac{dA^{(3)}} {d\p_3} & = &  0
 \nonumber \\
 \frac{dB^{(3)}} {d\p_3} & = &  (-\u_1^{(2,3)}/2/\sqrt{\p_2-\p_3} + \u_1^{(2,4)}/2/\sqrt{\p_3})
 \nonumber \\
 \frac{dD^{(3)}} {d\p_3} & = &    \frac{d\bar{a}_{0}^{(3)} } {d\p_3}
 \nonumber \\
 \frac{dF^{(3)}} {d\p_3} & = &    \frac{d\bar{a}_{2}^{(3)} } {d\p_3}.
  \end{eqnarray}
After solving the remaining integrals and combining (\ref{eq:3negprac24a0a0}) and (\ref{eq:CXpapp1eq3}), we find the following $\p_3$ analogue to (\ref{eq:Cpapp1eq7})
 \begin{eqnarray}\label{eq:CXpapp1eq7}
  \frac{d\hat{f}_2^{(3)}} {d\p_3}  =
  \frac{d\hat{f}_{21}\lp \bar{a}_0^{(3)}\rp } {d\p_3}
  +   \frac{d\hat{f}_{22}\lp \bar{a}_0^{(3)}\rp } {d\p_3}
   +    \frac{d\hat{f}_{23}\lp I_1(\cA^{(3)})\rp } {d\p_3}.
\end{eqnarray}
From (\ref{eq:Capp1eq1}), we then easily obtain
\begin{eqnarray}\label{eq:CXpapp1eq8}
 \frac{dI_L^{(3)}}{d\p_2} & = &
 \frac{\alpha}{\c_3}\mE_{{\mathcal U}_4} \lp \frac{1}{\mE_{{\mathcal U}_3} \lp \hat{f}_3^{(2)} \rp^{\frac{\c_3}{\c_2}}  } \mE_{{\mathcal U}_3}
 \lp  \lp   \frac{\c_3}{\c_2}\frac{1}{\hat{f}_2^{(3)}}  \frac{d\hat{f}_2^{(3)}} {d\p_3}  \rp
 \lp \hat{f}_2^{(3)}\rp^{\c_3/\c_2}  \rp
 \rp.
     \end{eqnarray}
Combining (\ref{eq:Capp1eq1}), (\ref{eq:Capp1eq1a0}), (\ref{eq:CXpapp1eq2}), (\ref{eq:CXpapp1eq7}), and (\ref{eq:CXpapp1eq8})  we arrive at
\begin{eqnarray}\label{eq:CXpapp1eq9}
  \frac{d  \bar{\psi}_{rd}^{(3)}}{d\p_3}
 & = &
  + \frac{1}{2}
\c_3\q_3
   -
 \frac{\alpha}{\c_3}\mE_{{\mathcal U}_4} \lp \frac{1}{\mE_{{\mathcal U}_3} \lp \hat{f}_3^{(2)} \rp^{\frac{\c_3}{\c_2}}  } \mE_{{\mathcal U}_3}
 \lp  \lp   \frac{\c_3}{\c_2}\frac{1}{\hat{f}_2^{(3)}}  \frac{d\hat{f}_2^{(3)}} {d\p_3}  \rp
 \lp \hat{f}_2^{(3)}\rp^{\c_3/\c_2}  \rp
 \rp.
     \end{eqnarray}

\subsection{$\nu$ derivative}

We trivially have
\begin{eqnarray}\label{eq:Cnuapp1eq2}
\frac{dI_R^{(3)} }{d\nu}& = &
 0,
     \end{eqnarray}
and recalling on (\ref{eq:3negprac24a0a0})  write
 \begin{eqnarray}\label{eq:Cnuapp1eq3}
  \frac{d\bar{a}_{0}^{(3)} } {d\nu} &  = & 0
\nonumber \\
 \frac{d\bar{a}_{2}^{(3)} } {d\nu} &  = & 1/\sqrt{2\nu}/\sqrt{1-\p_2}
 \nonumber \\
  \frac{d\hat{f}_{21}\lp \bar{a}_0^{(3)}\rp } {d\nu} &  = & \frac{\c_2}{4\gamma_{sq}^{(p)}} e^{\frac{\c_2}{4\gamma_{sq}^{(p)}}\nu}
\lp \frac{1}{2} - \frac{1}{2}\erfc\lp  \frac{ \bar{a}_0^{(3)}}{\sqrt{2}} \rp \rp
\nonumber \\
 \frac{d\hat{f}_{22}\lp \bar{a}_2^{(3)}\rp } {d\nu} & = &  -\frac{\c_2}{4\gamma_{sq}^{(p)}}  e^{-\frac{\c_2}{4\gamma_{sq}^{(p)}}\nu} \frac{1}{2}\erfc\lp  \frac{\bar{a}_2^{(2)}}{\sqrt{2}} \rp
 - e^{-\frac{\c_2}{4\gamma_{sq}^{(p)}}\nu}
\frac{1}{\sqrt{\pi}} e^{-\lp \frac{ \bar{a}_2^{(3)}}{\sqrt{2}}   \rp^2}  \frac{1}{\sqrt{2}}  \frac{d\bar{a}_2^{(3)} }{d\nu}
 \nonumber \\
  \frac{dD^{(3)}} {d\nu} & = &    \frac{d\bar{a}_{0}^{(3)} } {d\nu}
   \nonumber \\
  \frac{dF^{(3)}} {d\nu} & = &    \frac{d\bar{a}_{2}^{(3)} } {d\nu}.
\end{eqnarray}
 After solving the remaining integrals, a combination of (\ref{eq:3negprac24a0a0}) and (\ref{eq:Cnuapp1eq3}) then gives
 \begin{eqnarray}\label{eq:Cnuapp1eq7}
  \frac{d\hat{f}_2^{(3)}} {d\nu}  =
  \frac{d\hat{f}_{21}\lp \bar{a}_0^{(3)}\rp } {d\nu}
  +   \frac{d\hat{f}_{22}\lp \bar{a}_0^{(3)}\rp } {d\nu}
   +    \frac{d\hat{f}_{23}\lp I_1(\cA^{(3)})\rp } {d\nu}.
\end{eqnarray}
One then from (\ref{eq:Capp1eq1})  easily finds
\begin{eqnarray}\label{eq:Cnuapp1eq8}
 \frac{dI_L^{(2)}}{d\nu} & = &
 \frac{\alpha}{\c_3}\mE_{{\mathcal U}_4} \lp \frac{1}{\mE_{{\mathcal U}_3} \lp \hat{f}_3^{(2)} \rp^{\frac{\c_3}{\c_2}}  } \mE_{{\mathcal U}_3}
 \lp  \lp   \frac{\c_3}{\c_2}\frac{1}{\hat{f}_2^{(3)}}  \frac{d\hat{f}_2^{(3)}} {d\nu}  \rp
 \lp \hat{f}_2^{(3)}\rp^{\c_3/\c_2}  \rp
 \rp.
     \end{eqnarray}
Combining (\ref{eq:Capp1eq1}), (\ref{eq:Capp1eq1a0}), (\ref{eq:Cnuapp1eq2}), (\ref{eq:Cnuapp1eq7}), and (\ref{eq:Cnuapp1eq8}) we finally have
\begin{eqnarray}\label{eq:Cpapp1eq9}
  \frac{d  \bar{\psi}_{rd}^{(3)}}{d\nu}
 & = &
    -\frac{2-\alpha}{4\gamma_{sq}^{(p)}}
-
 \frac{\alpha}{\c_3}\mE_{{\mathcal U}_4} \lp \frac{1}{\mE_{{\mathcal U}_3} \lp \hat{f}_3^{(2)} \rp^{\frac{\c_3}{\c_2}}  } \mE_{{\mathcal U}_3}
 \lp  \lp   \frac{\c_3}{\c_2}\frac{1}{\hat{f}_2^{(3)}}  \frac{d\hat{f}_2^{(3)}} {d\nu}  \rp
 \lp \hat{f}_2^{(3)}\rp^{\c_3/\c_2}  \rp
 \rp.
     \end{eqnarray}

\subsection{$\q$ and $\gamma_{sq}^{(q)}$ derivative}

 All three $\q_2$, $\q_3$, and $\gamma_{sq}^{(q)}$ derivatives have been computed in \cite{Stojnicnegsphflrdt23}. For completeness we restate the obtained results
\begin{eqnarray}\label{eq:C2levder1}
   \frac{d\bar{\psi}_{rd}^{(3)}  }{d\q_2}
   & = &  (\c_2-\c_3)\Bigg(\Bigg. -\frac{\p_2}{2}
  +\frac{\q_2-\q_3}{2(2\gamma_{sq}^{(q)}-\c_2(1-\q_2))(2\gamma_{sq}^{(q)}-\c_2(1-\q_2)-\c_3(\q_2-\q_3))}
  \nonumber \\
& & +  \frac{\q_3}{2(2\gamma_{sq}^{(q)}-\c_2(1-\q_2)-\c_3(\q_2-\q_3))^2}   \Bigg.\Bigg)
\nonumber \\
   \frac{d\bar{\psi}_{rd}^{(3)}  }{d\q_3}
   & = &
    -\frac{1}{2}
\p_3\c_3
+  \frac{\c_3\q_3}{2(2\gamma_{sq}^{(q)}-\c_2(1-\q_2)-\c_3(\q_2-\q_3))^2}
\nonumber \\
   \frac{d\bar{\psi}_{rd}^{(3)} }{d\gamma_{sq}^{(q)}}
    & = &
    -1
-\Bigg(\Bigg. -\frac{1-\q_2}{2\gamma_{sq}^{(q)}(2\gamma_{sq}^{(q)}-\c_2(1-\q_2))}
 -\frac{\q_2-\q_3}{(2\gamma_{sq}^{(q)}-\c_2(1-\q_2))(2\gamma_{sq}^{(q)}-\c_2(1-\q_2)-\c_3(\q_2-\q_3))}
   \nonumber \\
& & -  \frac{\q_3}{(2\gamma_{sq}^{(q)}-\c_2(1-\q_2)-\c_3(\q_2-\q_3))^2}   \Bigg.\Bigg).
     \end{eqnarray}

\subsection{$\gamma_{sq}^{(p)}$ derivative}

Analogously to what happened on the second level, we now again observe
\begin{eqnarray} \label{eq:Cgamaapp1eq1}
   \frac{d\hat{f}_2^{(3)}}{d\gamma_{sq}^{(p)}}
   =-\frac{\c_2}{\gamma_{sq}^{(p)}}
   \frac{d\hat{f}_2^{(2)}}{d\c_2}.
 \end{eqnarray}
From (\ref{eq:Capp1eq1}) and (\ref{eq:Capp1eq1a0}) one then has
\begin{eqnarray} \label{eq:Cgamaapp1eq2}
  \frac{d  \bar{\psi}_{rd}^{(2)}}{\gamma_{sq}^{(p)}  }
 & = &
   1 + \frac{\nu(2-\alpha)}{4\lp\gamma_{sq}^{(p)}\rp^2}
-
 \frac{\alpha}{\c_3}\mE_{{\mathcal U}_4} \lp \frac{1}{\mE_{{\mathcal U}_3} \lp \hat{f}_3^{(2)} \rp^{\frac{\c_3}{\c_2}}  } \mE_{{\mathcal U}_3}
 \lp  \lp   \frac{\c_3}{\c_2}\frac{1}{\hat{f}_2^{(3)}}  \frac{d\hat{f}_2^{(3)}} {d\gamma_{sq}^{(p)}}  \rp
 \lp \hat{f}_2^{(3)}\rp^{\c_3/\c_2}  \rp
 \rp.
 \end{eqnarray}

\section{Proof of Corollary \ref{cor:3closedformrel1}}
\label{sec:appD}

\begin{proof}[Proof of Corollary \ref{cor:3closedformrel1}]
The first two equalities follow from (\ref{eq:C2levder1}) (detailed derivations are presented in equations (137)-(147) in \cite{Stojnicnegsphflrdt23}). We focus on the fourth equality. To facilitate the exposition we try to parallel the corresponding proof of Corollary \ref{cor:closedformrel1}. We first note that a combination of (\ref{eq:Cgamaapp1eq1})
and (\ref{eq:Cgamaapp1eq2}) gives
\begin{eqnarray} \label{eq:Daapp2eq1}
  \frac{d  \bar{\psi}_{rd}^{(2)}}{d\gamma_{sq}^{(p)}  }
 & = &
   1 + \frac{\nu(2-\alpha)}{4\lp\gamma_{sq}^{(p)}\rp^2}
+
\frac{\c_2}{\gamma_{sq}^{(p)}}
 \frac{\alpha}{\c_3}\mE_{{\mathcal U}_4} \lp \frac{1}{\mE_{{\mathcal U}_3} \lp \hat{f}_3^{(2)} \rp^{\frac{\c_3}{\c_2}}  } \mE_{{\mathcal U}_3}
 \lp  \lp   \frac{\c_3}{\c_2}\frac{1}{\hat{f}_2^{(3)}}  \frac{d\hat{f}_2^{(3)}} {d\c_2}  \rp
 \lp \hat{f}_2^{(3)}\rp^{\c_3/\c_2}  \rp
 \rp.
 \end{eqnarray}
Equalling the above derivative to zero implies
\begin{eqnarray} \label{eq:Daapp2eq2}
-\frac{\gamma_{sq}^{(p)}} {\c_2}\lp  1 + \frac{\nu(2-\alpha)}{4\lp\gamma_{sq}^{(p)}\rp^2} \rp
 & = &
 \frac{\alpha}{\c_3}\mE_{{\mathcal U}_4} \lp \frac{1}{\mE_{{\mathcal U}_3} \lp \hat{f}_3^{(2)} \rp^{\frac{\c_3}{\c_2}}  } \mE_{{\mathcal U}_3}
 \lp  \lp   \frac{\c_3}{\c_2}\frac{1}{\hat{f}_2^{(3)}}  \frac{d\hat{f}_2^{(3)}} {d\c_2}  \rp
 \lp \hat{f}_2^{(3)}\rp^{\c_3/\c_2}  \rp
 \rp.
 \end{eqnarray}
For $\bar{\psi}_{rd}^{(3)}=0$ (which is the critical condition for the injectivity capacity), one from (\ref{eq:Capp1eq1a0})  has
\begin{eqnarray}\label{eq:Daapp2eq3}
 \frac{\alpha}{\c_3}\mE_{{\mathcal U}_4}\log  \lp \mE_{{\mathcal U}_3} \lp \hat{f}_2^{(3)}\rp^{\frac{\c_3}{\c_2}}  \rp
  & = &
  \frac{1}{2}
(1-\p_2\q_2)\c_2   +  \frac{1}{2}
(\p_2\q_2-\p_3\q_3)\c_3     - I_R^{(3)}
  +\gamma_{sq}^{(p)} -\frac{\nu(2-\alpha)}{4\gamma_{sq}^{(p)}}.
     \end{eqnarray}
Equaling to zero the derivative in (\ref{eq:CXapp1eq9}) together with (\ref{eq:Daapp2eq3}) gives
\begin{align}\label{eq:Daapp2eq3a0}
 \frac{\alpha}{\c_2}\mE_{{\mathcal U}_4} \lp \frac{\mE_{{\mathcal U}_3}
 \lp  \lp  \log(\hat{f}_2^{(3)})  \rp
 \lp \hat{f}_2^{(3)}\rp^{\c_3/\c_2}  \rp}    {\mE_{{\mathcal U}_3} \lp \hat{f}_2^{(3)} \rp^{\frac{\c_3}{\c_2}}  }
 \rp
 & =
  \frac{\c_3}{2}
(\p_2\q_2-\p_3\q_3)
\nonumber \\
& \quad
-\c_3
\Bigg (\Bigg.
\frac{1}{2\c_3^2}\log \lp  \frac{(2\gamma_{sq}^{(q)}-\c_2(1-\q_2) -\c_3(\q_2-\q_3))}{(2\gamma_{sq}^{(q)}-\c_2(1-\q_2))}  \rp
\nonumber \\
& \quad
 + \frac{(\q_2-\q_3)}{2\c_3(2\gamma_{sq}^{(q)}-\c_2(1-\q_2) -\c_3(\q_2-\q_3))}
\nonumber \\
& \quad
+\frac{(\q_2-\q_3)\q_3}{2(2\gamma_{sq}^{(q)}-\c_2(1-\q_2)-\c_3(\q_2-\q_3))^2}
 \Bigg. \Bigg )
\nonumber \\
 & \quad
  \frac{1}{2}
(1-\p_2\q_2)\c_2   +  \frac{1}{2}
(\p_2\q_2-\p_3\q_3)\c_3   - I_R^{(3)}
  +\gamma_{sq}^{(p)} -\frac{\nu(2-\alpha)}{4\gamma_{sq}^{(p)}}.
     \end{align}
Equalling to zero the derivative in  (\ref{eq:Capp1eq9}) and utilizing ( \ref{eq:Daapp2eq2}) and (\ref{eq:Daapp2eq3a0}) , we further write
\begin{eqnarray}\label{eq:Daapp2eq3a1}
0
& = &
  \frac{\c_2}{2}
(1-\p_2\q_2)
 -\c_2
\Bigg (\Bigg.
\frac{1}{2\c_2^2}\log\lp \frac{ 2\gamma_{sq}^{(q)} -\c_2(1-\q_2)}   {2\gamma_{sq}^{(q)} } \rp
+ \frac{1-\q_2}{2\c_2(2\gamma_{sq}^{(q)}-\c_2(1-\q_2))} \nonumber \\
& &
+ \frac{1-\q_2}{2\c_3(2\gamma_{sq}^{(q)}-\c_2(1-\q_2) -\c_3(\q_2-\q_3))}
- \frac{1-\q_2}{2\c_3(2\gamma_{sq}^{(q)}-\c_2(1-\q_2))}
\nonumber \\
& &
+
\frac{(1-\q_2)\q_3}{2(2\gamma_{sq}^{(q)}-\c_2(1-\q_2)-\c_3(\q_2-\q_3))^2}
 \Bigg. \Bigg )
\nonumber \\
 & &
+  \frac{\c_3}{2}
(\p_2\q_2-\p_3\q_3)
 -\c_3
\Bigg (\Bigg.
\frac{1}{2\c_3^2}\log \lp  \frac{(2\gamma_{sq}^{(q)}-\c_2(1-\q_2) -\c_3(\q_2-\q_3))}{(2\gamma_{sq}^{(q)}-\c_2(1-\q_2))}  \rp
\nonumber \\
& &
 + \frac{(\q_2-\q_3)}{2\c_3(2\gamma_{sq}^{(q)}-\c_2(1-\q_2) -\c_3(\q_2-\q_3))}
+\frac{(\q_2-\q_3)\q_3}{2(2\gamma_{sq}^{(q)}-\c_2(1-\q_2)-\c_3(\q_2-\q_3))^2}
 \Bigg. \Bigg )
\nonumber \\
 & &
  \frac{1}{2}
(1-\p_2\q_2)\c_2   +  \frac{1}{2}
(\p_2\q_2-\p_3\q_3)\c_3   - I_R^{(3)}
  +\gamma_{sq}^{(p)} -\frac{\nu(2-\alpha)}{4\gamma_{sq}^{(p)}}
  \nonumber \\
& &
+
\gamma_{sq}^{(p)}\lp  1 + \frac{\nu(2-\alpha)}{4\lp\gamma_{sq}^{(p)}\rp^2} \rp
\nonumber \\
& = &
  \c_2
(1-\p_2\q_2)
 -\c_2
\Bigg (\Bigg.
  \frac{1-\q_2}{2\c_2(2\gamma_{sq}^{(q)}-\c_2(1-\q_2))}
+ \frac{1-\q_2}{2\c_3(2\gamma_{sq}^{(q)}-\c_2(1-\q_2) -\c_3(\q_2-\q_3))}
  \nonumber \\
& &
- \frac{1-\q_2}{2\c_3(2\gamma_{sq}^{(q)}-\c_2(1-\q_2))}
+
\frac{(1-\q_2)\q_3}{2(2\gamma_{sq}^{(q)}-\c_2(1-\q_2)-\c_3(\q_2-\q_3))^2}
 \Bigg. \Bigg )
\nonumber \\
 & &
+  \c_3
(\p_2\q_2-\p_3\q_3)
\nonumber \\
& &
-\c_3
\Bigg (\Bigg.
   \frac{(\q_2-\q_3)}{2\c_3(2\gamma_{sq}^{(q)}-\c_2(1-\q_2) -\c_3(\q_2-\q_3))}
+ \frac{(\q_2-\q_3)\q_3}{2(2\gamma_{sq}^{(q)}-\c_2(1-\q_2)-\c_3(\q_2-\q_3))^2}
 \Bigg. \Bigg )
\nonumber \\
 & &
  -\gamma_{sq}^{)q)}  +2\gamma_{sq}^{(p)}
-\frac{\q_3}{2(2\gamma_{sq}^{(q)}-\c_2(1-\q_2)-\c_3(\q_2-\q_3))}
  \nonumber \\
  & = &
  \c_2
(1-\p_2\q_2)
 -
   \frac{\c_2(1-\q_2)}{2\c_2(2\gamma_{sq}^{(q)}-\c_2(1-\q_2))}
- \frac{\c_2(1-\q_2)+\c_3(\q_2-\q_3)}{2\c_3(2\gamma_{sq}^{(q)}-\c_2(1-\q_2) -\c_3(\q_2-\q_3))}
  \nonumber \\
& &
+ \frac{\c_2(1-\q_2)}{2\c_3(2\gamma_{sq}^{(q)}-\c_2(1-\q_2))}
-
\frac{\c_2(1-\q_2)\q_3+\c_3(\q_2-\q_3)\q_3}{2(2\gamma_{sq}^{(q)}-\c_2(1-\q_2)-\c_3(\q_2-\q_3))^2}
 \nonumber \\
 & &
+  \c_3
(\p_2\q_2-\p_3\q_3)
 \nonumber \\
 & &
  -\gamma_{sq}^{)q)}  +2\gamma_{sq}^{(p)}
  -\frac{\q_3}{2(2\gamma_{sq}^{(q)}-\c_2(1-\q_2)-\c_3(\q_2-\q_3))}.
     \end{eqnarray}
Equaling to zero  the derivatives in (\ref{eq:C2levder1}), we find
\begin{eqnarray}\label{eq:Daapp2eq3a2}
   \frac{\p_2-\p_3} {\q_2-\q_3}
   & = &  \frac{1}{(2\gamma_{sq}^{(q)}-\c_2(1-\q_2))(2\gamma_{sq}^{(q)}-\c_2(1-\q_2)-\c_3(\q_2-\q_3))}
 \nonumber \\
\frac{\p_3}{\q_3}
   & = &
   \frac{1}{(2\gamma_{sq}^{(q)}-\c_2(1-\q_2)-\c_3(\q_2-\q_3))^2}
\nonumber \\
   \frac{1-\p_2} {1-\q_2}
    & = &
  \frac{1}{2\gamma_{sq}^{(q)}(2\gamma_{sq}^{(q)}-\c_2(1-\q_2))}.
     \end{eqnarray}
Combining (\ref{eq:Daapp2eq3a1}) and (\ref{eq:Daapp2eq3a2}) we then have
\begin{eqnarray}\label{eq:Daapp2eq3a3}
0
    & = &
  \c_2
(1-\p_2\q_2)
 - \gamma_{sq}^{(q)}(1-\p_2)
- \frac{\c_2(1-\q_2)+\c_3(\q_2-\q_3)}{2\c_3(2\gamma_{sq}^{(q)}-\c_2(1-\q_2) -\c_3(\q_2-\q_3))}
  \nonumber \\
& &
+ \frac{\c_2(1-\q_2)}{2\c_3(2\gamma_{sq}^{(q)}-\c_2(1-\q_2))}
-
\frac{\c_2(1-\q_2)\q_3+\c_3(\q_2-\q_3)\q_3}{2(2\gamma_{sq}^{(q)}-\c_2(1-\q_2)-\c_3(\q_2-\q_3))^2}
 \nonumber \\
 & &
+  \c_3
(\p_2\q_2-\p_3\q_3)
 \nonumber \\
 & &
  -\gamma_{sq}^{)q)}  +2\gamma_{sq}^{(p)}
  -\frac{\q_3}{2(2\gamma_{sq}^{(q)}-\c_2(1-\q_2)-\c_3(\q_2-\q_3))}
  \nonumber \\
      & = &
  \c_2
(1-\p_2\q_2)
 - \gamma_{sq}^{(q)}(1-\p_2)
- \frac{\gamma_{sq}^{(q)}(\q_2-\q_3)}{(2\gamma_{sq}^{(q)}-\c_2(1-\q_2))(2\gamma_{sq}^{(q)}-\c_2(1-\q_2) -\c_3(\q_2-\q_3))}
  \nonumber \\
& &
 -
\gamma_{sq}^{(q)}\p_3 + \frac{1}{2}\sqrt{\p_3\q_3}
+  \c_3
(\p_2\q_2-\p_3\q_3)
   -\gamma_{sq}^{)q)}  +2\gamma_{sq}^{(p)}  -\frac{1}{2}\sqrt{\p_3\q_3}
   \nonumber \\
      & = &
  \c_2
(1-\p_2\q_2)
 - \gamma_{sq}^{(q)}(1-\p_2)
- \gamma_{sq}^{(q)}(\p_2-\p_3)
  \nonumber \\
& &
 -
\gamma_{sq}^{(q)}\p_3
+  \c_3
(\p_2\q_2-\p_3\q_3)
   -\gamma_{sq}^{)q)}  +2\gamma_{sq}^{(p)}
   \nonumber \\
      & = &
  \c_2
(1-\p_2\q_2)
 +  \c_3
(\p_2\q_2-\p_3\q_3)
   -2\gamma_{sq}^{)q)}  +2\gamma_{sq}^{(p)}.
     \end{eqnarray}
From (\ref{eq:Daapp2eq3a2}) we also find
\begin{eqnarray}\label{eq:Daapp2eq3a4}
\gamma_{sq}^{(q)}
& = &
\frac{1}{2}\frac{1-\q_2}{1-\p_2} \frac{\p_2-\p_3}{\q_2-\q_3}\sqrt{\frac{\q_3}{\p_3}}
\nonumber \\
 \c_2    & = &
  \frac{2\gamma_{sq}^{(q)}}{ (1-\q_2) }-  \frac{1}{2\gamma_{sq}^{(q)} (1-\p_2) }
=
 \frac{1}{1-\p_2} \frac{\p_2-\p_3}{\q_2-\q_3}\sqrt{\frac{\q_3}{\p_3}}
-
 \frac{1}{1-\q_2} \frac{\q_2-\q_3}{\p_2-\p_3}\sqrt{\frac{\p_3}{\q_3}}
  \nonumber \\
\c_3    & = &
-  \frac{2\gamma_{sq}^{(q)}(1-\p_2)  } { (\p_2-\p_3) (1-\q_2)}+  \frac{1-\q_2}{ (2\gamma_{sq}^{(q)} (1-\p_2 ))(\q_2-\q_3) }
=
\frac{1}{\p_2-\p_3}\sqrt{\frac{\p_3}{\q_3}}
-\frac{1}{\q_2-\q_3}\sqrt{\frac{\q_3}{\p_3}}.
     \end{eqnarray}
Combining (\ref{eq:Daapp2eq3a3}) and  (\ref{eq:Daapp2eq3a4}) one further has
\begin{eqnarray}\label{eq:Daapp2eq3a5}
0
        & = &
  \c_2
(1-\p_2) +\c_2(\p_2-\p_2\q_2)
+  \c_3
(\p_2\q_2-\p_3\q_3)
   -2\gamma_{sq}^{)q)}  +2\gamma_{sq}^{(p)}
\nonumber \\
        & = &
  \frac{\p_2-\p_3}{\q_2-\q_3}\sqrt{\frac{\q_3}{\p_3}}
 -\frac{1}{2\gamma_{sq}^{(q)}}
 + 2\gamma_{sq}^{(q)}\p_2
- \frac{\q_2-\q_3}{\p_2-\p_3}\sqrt{\frac{\p_3}{\q_3}}\p_2
 +  \c_3
(\p_2\q_2-\p_3\q_3)
   -2\gamma_{sq}^{)q)}  +2\gamma_{sq}^{(p)}
\nonumber \\
        & = &
  \frac{\p_2-\p_3}{\q_2-\q_3}\sqrt{\frac{\q_3}{\p_3}}
 -\frac{1}{2\gamma_{sq}^{(q)}}
 + 2\gamma_{sq}^{(q)}\p_2
- \frac{\q_2-\q_3}{\p_2-\p_3}\sqrt{\frac{\p_3}{\q_3}}\p_2
  \nonumber \\
& & +  \c_3\q_2(\p_2-\p_3) +\c_3\p_3(\q_2-\q_3)
   -2\gamma_{sq}^{)q)}  +2\gamma_{sq}^{(p)}
\nonumber \\
        & = &
  \frac{\p_2-\p_3}{\q_2-\q_3}\sqrt{\frac{\q_3}{\p_3}}
 -\frac{1}{2\gamma_{sq}^{(q)}}
 + 2\gamma_{sq}^{(q)}\p_2
- \frac{\q_2-\q_3}{\p_2-\p_3}\sqrt{\frac{\p_3}{\q_3}}\p_2
 \nonumber \\
& & +  \lp \frac{1}{\p_2-\p_3}\sqrt{\frac{\p_3}{\q_3}}
-\frac{1}{\q_2-\q_3}\sqrt{\frac{\q_3}{\p_3}}  \rp \q_2(\p_2-\p_3)
\nonumber \\
& & +
\lp  \frac{1}{\p_2-\p_3}\sqrt{\frac{\p_3}{\q_3}}
-\frac{1}{\q_2-\q_3}\sqrt{\frac{\q_3}{\p_3}}   \rp   \p_3(\q_2-\q_3)
   -2\gamma_{sq}^{)q)}  +2\gamma_{sq}^{(p)}
\nonumber \\
        & = &
  -\frac{1}{2\gamma_{sq}^{(q)}}
- (\q_2-\q_3)\sqrt{\frac{\p_3}{\q_3}}
    +   \sqrt{\frac{\p_3}{\q_3}} \q_2
 +\lp
-\frac{1}{\q_2-\q_3}\sqrt{\frac{\q_3}{\p_3}}   \rp   \p_3(\q_2-\q_3)
     +2\gamma_{sq}^{(p)}
\nonumber \\
        & = &
  -\frac{1}{2\gamma_{sq}^{(q)}}
- (-\q_3)\sqrt{\frac{\p_3}{\q_3}}
   +\lp
-\frac{1}{\q_2-\q_3}\sqrt{\frac{\q_3}{\p_3}}   \rp   \p_3(\q_2-\q_3)
     +2\gamma_{sq}^{(p)}
\nonumber \\
        & = &
  -\frac{1}{2\gamma_{sq}^{(q)}}
        +2\gamma_{sq}^{(p)}.
             \end{eqnarray}
Keeping in mind that $\hat{\gamma}_{sq}^{(q)}$ and $\hat{\gamma}_{sq}^{(p)}$ are precisely the solutions of zero-derivatives equations, one after connecting beginning and end in the last sequence of equalities completes the proof.
\end{proof}

\end{document}